\pdfoutput = 1

\pdfoutput=1

\documentclass[11pt]{article}
\usepackage{fullpage}
\usepackage{comment}
\usepackage[utf8]{inputenc} 
\usepackage[T1]{fontenc}    
\usepackage{hyperref}       
\usepackage{url}            
\usepackage{booktabs}       
\usepackage{amsfonts}       
\usepackage{nicefrac}       
\usepackage{microtype}      
\usepackage{xcolor}         
\usepackage[pdftex]{graphicx}
\usepackage{subcaption}
\usepackage{caption}

\usepackage{amsfonts}       
\usepackage{amssymb}
\usepackage{amsmath}
\usepackage{amsthm}
\usepackage{mathtools}
\usepackage{multicol}
\usepackage{multirow}
\usepackage{amsthm}
\newtheorem{lem}{Lemma}

\theoremstyle{definition}
\newtheorem{cor}{Corollary}
\newtheorem{defn}{Definition}
\newtheorem{rem}{Remark}
\usepackage{color}

\usepackage{array}
\usepackage{algorithm}
\usepackage{algorithmic}

\newcommand{\trace}{\mathrm{tr}}

\usepackage{soul}
\newcounter{guocomm}

\newcommand{\guorm}[1]{\ignorespaces}

\usepackage[utf8]{inputenc}

\title{\textbf{Frameless Graph Knowledge Distillation}}

\author{Dai Shi,\footnote{Western Sydney Univeristy,
\texttt{20195423@student.westernsydney.edu.au}, \texttt{y.guo@westernsydney.edu.au}} \, 
Zhiqi Shao\footnote{University of Sydney,
\texttt{zsha2911@uni.sydney.edu.au}, \texttt{junbin.gao@sydney.edu.au} \\
\text{\quad \,\,} Dai Shi and Zhiqi Shao are with equal contribution.}, \,
Yi Guo, \, Junbin Gao}  
\date{}

\begin{document}

\maketitle

\begin{abstract}
Knowledge distillation (KD) has shown great potential for transferring knowledge from a complex teacher model to a simple student model in which the heavy learning task can be accomplished efficiently and without losing too much prediction accuracy.  Recently, many attempts have been made by applying the KD mechanism to the graph representation learning models such as graph neural networks (GNNs) to accelerate the model's inference speed via student models. However, many existing KD-based GNNs utilize MLP as a universal approximator in the student model to imitate the teacher model's process without considering the graph knowledge from the teacher model. In this work, we provide a KD-based framework on multi-scaled GNNs, known as graph framelet, and prove that by adequately utilizing the graph knowledge in a multi-scaled manner provided by graph framelet decomposition, the student model is capable of adapting both homophilic and heterophilic graphs and has the potential of alleviating the over-squashing issue with a simple yet effectively graph surgery. Furthermore, we show how the graph knowledge supplied by the teacher is learned and digested by the student model via both algebra and geometry. Comprehensive experiments show that our proposed model can generate learning accuracy identical to or even surpass the teacher model while maintaining the high speed of inference. 


\end{abstract}

\section{Introduction}
Graph neural networks (GNNs) have achieved supreme performance for the graph-structured data in terms of both node classification and graph-level prediction \cite{wu2020comprehensive,ji2021survey}.  There are in general two types of GNNs. One type is spatial GNNs, such as MPNN \cite{gilmer2017neural}, GAT \cite{velivckovic2017graph}, and GIN \cite{xu2018powerful}. They produce node feature representation by aggregating neighbouring information in graph spatial domain. Another type is spectral GNN, for example, ChebNet \cite{defferrard2016convolutional} and GCN \cite{kipf2016semi}, generating node representation by applying filter functions in the spectral domain associated with a graph. Among enormous GNNs models, the graph wavelet frames (known as graph framelet) \cite{dong2017sparse,hammond2011wavelets,li2020fast,wang2021deep,xu2018graph,zheng2021framelets,zheng2020mathnet,zheng2022decimated,zhang2022ms} provide a multi-resolution analysis on graph signals to sufficiently capture the node feature information at different scales. Graph framelets have shown remarkable empirical outcomes in various graph learning tasks due to their \textit{tightness} property, which guarantees to reconstruct the filtered graph signals from multiple (generally known as low-pass and high-pass) domains. Although many benefits have been shown from applying GNNs to graph structured data, several issues have been identified recently for GNNs. In general, there are two sources of these issues. One of them is raised when a GNN contains a relatively large number of layers, leading to the so-called over-smoothing \cite{li2018deeper} and over-squashing  problems\cite{topping2021understanding}. The other source of issues is mainly from the adaptability of GNNs to datasets with different properties,  for example, homophilic and heterophilic graphs, two ends of the spectrum of the consistency of  nodes labels in neighbourhoods \cite{chen2021graph}. Another example is a model's efficiency, i.e. inference latency, more precisely, how well it  deals with large-scale graph input \cite{zhang2021graph}.


Although many attempts have been made to tackle the first three issues, i.e., over-smoothing, over-squashing, and heterophily adaption, searching the solution to alleviating inference latency has just begun recently due to the increasing press from large scale graph learning tasks in industry. The first brick that paves the path of accelerating the model's inference speed is graph pruning \cite{zhou2021accelerating} and quantization \cite{tailor2020degree}. However, it is known that the effectiveness of these methods is very limited as graph dependency is not yet resolved. 
To go around this problem, Zhang \emph{et al.}  \cite{zhang2021graph} provided a graph-free method to transfer  GNNs model into a multi-layer perceptron (MLP). The transformation method in \cite{zhang2021graph}, known as graph knowledge distillation (KD), aims to train a student model that has simpler model structure or less computation complexity with the knowledge presented by the more complex teacher model. The approach was first developed in computer vision \cite{hinton2015distilling} and showed good performance in delivering identical or even higher learning accuracy from the student model \cite{zhang2021graph,hinton2015distilling,chen2022sa,zhang2018deep,chen2020online,sun2019patient,xu2020deep,zhu2020knowledge}. 

The success of the method proposed in \cite{zhang2021graph} relies heavily on the relationship between the node features and graph topology. When most connected nodes are in the same class (known as high homophilic), mixing node features with adjacency information and producing a graph-less model is an effective strategy. However, when the graph connectivity and the node classes are almost independent (i.e., the graph is heterophilic), one may have to further refine the model to capture more detailed graph knowledge. Recently, the work in \cite{chen2022sa} takes this point and builds a structure-aware {MLP}. In fact, it has been shown that whether a student model can deliver a supreme performance depends on properties of data and training process, such as the geometry of the data, the bias in the 
distillation objective function, and strong monotonicity of the student classifier \cite{gou2021knowledge}. This inspires us to distill the knowledge from a multi-scale graph signal decomposition, i.e., framelet decomposition, to pass onto the student model. Moreover, we further simplify the original graph framelet model so that its corresponding student model can capture multi-hop graph information and thus gain a higher imitation power to the computation conducted in the later layer of the teacher model. We summarize our main contribution in this paper as follows: 
\begin{itemize}
    \item  We propose a new knowledge distillation process to learn from the multi-scale graph neural networks (known as framelets). 
    Our student model can accordingly capture detailed graph information in a multi-scale way. 
    \item We analyze how our student model can adapt to both homophilic and heterophilic graph datasets, which is followed by a detailed anatomy of how the student models learn and digest the teacher's knowledge during distillation. 
    \item We also propose a new simplified framelet model as another teacher model and provide a detailed discussion on the difference between the two teacher models from both energy and computational perspectives. For this new model, we show that the so-called over-squashing issue is bounded, and a graph geometry-based surgery is provided to further alleviate it when the number of layers in the teacher model is high. 
    \item We conduct comprehensive experiments to verify our claims in the paper. Experimental results show that the proposed model efficiently achieves state-of-the-art outcomes for node classification learning tasks.    
\end{itemize}

The rest of the paper is organized as follows: Section~\ref{related_works} provides a literature review on the recent development of framelet-based graph learning and knowledge distillation. In Section~\ref{preliminaries}, we provide preliminaries on the necessary components of our method, such as knowledge distillation, graph, and graph framelet. We propose our distillation models in Section~\ref{propose_model}  with some additional discussions including the difference between teacher models, what are the potential problems via computation, how student models learn and further handle the graph knowledge. We provide our experimental results and analysis in Section~\ref{experiment} to verify our claims in the paper. Lastly, the paper is concluded in Section~\ref{Conclusion}.

\section{Related Works}\label{related_works}
\subsection{Framelet-based Graph Learning}
Graph framelet is a type of wavelet frames. The work in \cite{dong2017sparse} presents an efficient way of approximating the graph framelet transform through the Chebyshev polynomials. Based on \cite{dong2017sparse}, many framelet-based graph learning methods have been developed recently. For example,  \cite{zheng2020mathnet,yang2022quasi} conducted graph framelet transform in graph (framelet filtered) spectral domain, with superior performances in both node classification and graph-level prediction (i.e., graph pooling) tasks. Later on, the graph framelet was applied to graph denoising \cite{zhou2021graph}, dynamic graph \cite{zhou2021spectral} and regularization based graph learning scheme \cite{shao2022generalized}. Rather than exploring the application of framelet on the spectral domain, recently, spatial framelet was proposed in \cite{chendirichlet} to overcome the so-called over-smoothing problem.  A (Dirichlet) energy perturbation scheme was introduced to enhance framelet model's adaption to heterophily graph dataset. Recently, \cite{han2022generalized} presented an asymptotic analysis of the framelet model to show graph framelet is capable of adapting both homophily and heterophily graphs.

\subsection{Knowledge Distillation}
Knowledge distillation is first introduced in the field of computer vision \cite{hinton2015distilling} in order to learn the knowledge contained in a relatively large-scale neural network (known as the teacher model) by a relatively small network (the student model) without losing too much prediction power. The learning schemes of knowledge distillation can be generally classified into two categories. The first is the offline distillation \cite{hinton2015distilling,chen2022sa,zhang2021graph,ye2021distillation} where the student model maximally approximate the learning outcome obtained from a pre-trained teacher model. The second is the online distillation including \cite{zhang2018deep,chen2020online}. Both teacher model and student model do the learning task simultaneously, and the whole distillation process is end-to-end trainable. In addition, self-distillation, where student and teacher share the same model \cite{zhang2019your,hou2019learning}, falls into this category. Good performance of self-distillation models was reported in \cite{zhang2019your,lee2019rethinking}. Some interesting distillation schemes such as the one combining both self and online distillation \cite{sun2019patient} have been developed recently.

\section{Preliminaries}\label{preliminaries}     

\subsection{Graph and Message Passing Neural Networks}
A graph $\mathcal G$ 
of $n$ nodes can be represented by its adjacency matrix $\mathbf A\in \mathbb R^{n \times n}$ , where the $i,j$th element $a_{i,j}$ indicates the connectivity from node $i$ to node $j$. Throughout this paper, we assume $\mathcal G$ is undirected (i.e., $a_{i,j} = a_{j,i}$). Let $\widehat {\mathbf A} = {\mathbf D}^{-1/2} (\mathbf A+ \mathbf I){\mathbf D}^{-1/2}$ be the normalized adjacency matrix where ${\mathbf D}$ is the diagonal matrix and its $ii$th entry is the sum of the $i$th row of $\mathbf A+ \mathbf I$. In addition, $\widehat{\mathbf L} = \mathbf{I-\widehat A}$ is the normalized Laplacian matrix. From spectral graph theory \cite{chung1997spectral}, we have $\widehat{\mathbf L} \succeq 0$, i.e. $\widehat{\mathbf L}$ is a positive semi-definite (SPD) matrix. Let $\lambda_i$ for $i=1, ..., n$ be the eigenvalues of $\widehat{\mathbf L}$, also known as graph spectra. We have $\lambda_i \in [0,2]$ . 

Messaging-passing neural networks (MPNNs) is first introduced in 
\cite{gilmer2017neural} as a unified framework for many GNNs models, including graph convolution networks (GCN)\cite{kipf2016semi}, graph attention neural networks \cite{velivckovic2017graph}(GAT), and graph SAGE \cite{hamilton2017inductive}. The principle of MPNNs is to generate graph node representation based on the graph node feature and adjacency information. Let $\mathbf X \in \mathbb R^{n \times d_0}$ be the node feature matrix of $\mathcal G$, and let $\mathbf H^{(\ell)}\in \mathbb R^{n \times d_{\ell}}$ be the embedded feature matrix at layer $\ell$ and additionally we have $\mathbf H^{(0)} = \mathbf X$. For any individual node $i$ (or indexed by node notation $v$), we write $\mathbf h^{(\ell)}_i$ as its feature at layer $\ell$ (i.e., the $i$-th row of $\mathbf H^{(\ell)}$), then we have the general formulation of MPNNs as: 
\begin{align}\label{mpnn}
    \mathbf h^{(\ell+1)}_i = \phi_\ell \left(\bigoplus_{j\in \mathcal N_i} \psi_\ell(\mathbf h^{(\ell)}_i,\mathbf h^{(\ell)}_j) \right),
\end{align}
where $\mathcal N_i$ stands for the set of all neighbors of node $i$, $\phi_\ell$ is the activation function, $\bigoplus$ is the aggregation operator and $\psi_\ell$ is the message passing function. As a typical example of  MPNN, the graph convolution networks (GCN) defines a layer-wise propagation rule based on the graph adjacency information as:
\begin{align}
     \mathbf H^{(\ell + 1)} = \sigma \big( \widehat{\mathbf A} \mathbf H^{(\ell)} \mathbf W^{(\ell)}  \big), \label{eq_classic_gcn}
\end{align}
with $\mathbf W^{(\ell)} \in \mathbb{R}^{d_{\ell}\times d_{\ell+1}}$ being the learnable feature transformation at layer $\ell$ and $\sigma()$ an activation function. Let $\{ (\lambda_i, \mathbf u_i) \}_{i=1}^n$ be the set of eigenvalue and eigenvector pairs of $\widehat{\mathbf L}$, i.e., $\widehat{\mathbf L} = \mathbf U^\top \mathbf \Lambda \mathbf U$ with $\mathbf U=[\mathbf u_1, ..., \mathbf u_n]$ and $\Lambda =\text{diag}[\lambda_1, ..., \lambda_n]$. Each column of $\mathbf H^{(\ell)}$ can be explained as a graph signal. It is easy to verify that the key operation in GCN corresponds to a localized filter by the graph Fourier transform, i.e., $\mathbf h^{(\ell+1)} = \mathbf U^\top (\mathbf I_n - \mathbf \Lambda) \mathbf U \mathbf h^{(\ell)}$, where we still use $\mathbf h^{(\ell)}$ to represent the feature transformed signals from  $\mathbf h^{(\ell)}$ with $\mathbf W^{(\ell)}$. 
In fact, $\mathbf U \mathbf h$ is known as the Fourier transform of a graph signal $\mathbf h \in \mathbb R^n$. For convenience reasons, in the sequel, we will use MPNNs and GNNs interchangeably.

\subsection{Graph Knowledge Distillation}
Knowledge distillation is initially designed as a model compression and acceleration method to efficiently produce nearly identical or even superior learning outcomes from relatively simple\footnote{The ``simplicity'' we refer here is either with a small number of computational steps (compression, i.e., self-distillation \cite{zhang2021self}) or some straightforward structure such as MLP. 
} 
student model that is designed to learn the knowledge  from a complicated teacher model \cite{hinton2015distilling}. Let us put it through the lens of graph learning. As shown in Eq.~\eqref{mpnn}, 
two pieces of information are available, graph adjacency ($\widehat{\mathbf A}$) and node features ($\mathbf X$) or their variants. They need to be provided for student models, say simple GNN or MLP, to learn.
Therefore, in general, the graph knowledge distillation process can be presented as follows: 
\begin{align}\label{distillation_initial}
    \mathcal D(\zeta_\mathcal D(\widehat{\mathbf A}),\xi_\mathcal D(\mathbf X)) \approx \text{GNN}(\zeta_T(\widehat{\mathbf A}), \xi_T(\mathbf X)),
\end{align}
where $\mathcal D$ stands for the general student model that aims to produce the learning outcomes approximating those of the teacher, and $\zeta_\mathcal D$ and $\xi_\mathcal D$ are the functions that are applied to the graph source knowledge (i.e.,$\widehat{\mathbf A}$ and $\mathbf X$). Similarly, the teacher model, some {GNN} model, can also be equipped with $\zeta_T$ and $\xi_T$ to produce $\widehat{\mathbf A}$ and $\mathbf X$'s variants in order to achieve better performance. For example, when $\zeta_T(\widehat{\mathbf A}) = \mathbf S \odot \widehat{\mathbf A}$ where $\mathbf S \in \mathbb  R^{n \times n}$ is the matrix with entries of attention coefficients generated from some attention mechanism \cite{velivckovic2017graph}, then we have reached the  GAT. We conjure that that one of the keys to success for a knowledge distillation model is to present the student model with sufficient graph information no matter whether such information is pre-processed or further augmented. In addition, as in Eq.~\eqref{distillation_initial}, a simpler learning model $ \mathcal D(\zeta_\mathcal D(\widehat{\mathbf A}),\xi_\mathcal D(\mathbf X))$ is applied to approximate a GNN. One common training strategy for knowledge distillation is to minimize the Kullback–Leibler divergence (KL-divergence) between the logits of the teacher and the student. 

\begin{rem}[Universal approximator]\label{Universal_approximation}
To ensure that student model is capable of imitating the performance of teacher model and further producing identical or even better learning outcomes than the teacher model, the student model has to be a universal approximator. As a widely recognised universal approximator, {MLP} has been applied to many distillation applications \cite{zhang2021graph,chen2022sa}. {MLP} is also one of the most commonly applied structures in many deep neural network (DNN) architectures. It is well known for capable of approximating a large family of functions, e.g. $L^2$, the group of all second-order diffeomorphisms. We note that there are many universal approximators other than {MLPs}. For example, one can approximate a function from $L^2$ with a finite composition of diffeomorphisms with compact support, which can be further approximated by a finite composition of flow mappings, etc. Refer to \cite{khesin2014arnold} for more details. This is potentially a deeper explanation of the approximation power of those distillation models utilizing other universal approximator rather than MLP although such discussion is out of the scope of this paper. 
\end{rem}

\subsection{Graph Framelet Convolution and Its Simplification}
Now we turn to framelets GNN, our main focus here. It has some interesting properties and hence is chosen as a role model. 
The graph framelets are defined by a set of filter banks  $\eta_{a,b} = \{ a; b^{(1)}, ..., b^{(L)} \}$ and the complex-valued scaling functions induced from it. The scaling functions are usually denoted as $\Psi = \{ {\alpha}; \beta^{(1)}, ..., \beta^{(L)} \}$ where $L$ represents the number of high-pass filters. Specifically, the framelet framework requires the following (refining) relationship  between the scaling functions and filter banks: 
\begin{align*}
\widehat{\alpha}(2\xi) = \widehat{a}(\xi) \widehat{\alpha}(\xi)   \,\,\, \text{and} \,\,\, &\widehat{\beta^{(r)}}(2\xi) = \widehat{b^{(r)}}(\xi) \widehat{\alpha}(\xi) ,\\
& \,\,\, \forall \xi \in \mathbb R, r =1,...,L 
\end{align*}
where  $\widehat{\alpha}$ and $\widehat{\beta^{(r)}}$ are the Fourier transformation of $\alpha$ and $\beta^{(r)}$, and  $\widehat{a}$, $\widehat{b^{(r)}}$ are the corresponding Fourier series of $a$, $b^{(r)}$ respectively. Then the graph framelets can be further  defined by $\varphi_{j,p}(v) = \sum_{i=1}^n \widehat{\alpha}\big( {\lambda_i}/{2^j} \big) u_i(p) u_i(v)$ and $\psi^r_{j,p}(v) = \sum_{i = 1}^n \widehat{\beta^{(r)}} \big( \lambda_i/ 2^j \big) u_i(p) u_i(v)$ for $r = 1,..., L$ and scale level $j = 1,...,J$. We use $u_i(v)$ to represent the eigenvector $\mathbf u_i$ at node $v$. $\varphi_{j,p}(\cdot)$ and $\psi^r_{j,p}(\cdot)$ are known as the \textit{low-pass framelets} and \textit{high-pass framelets} at node $p$.

Based on the notation presented above, one can define the framelet decomposition matrices $\mathcal W_{0,J}$ and $\mathcal W_{r,J}$  for a multi-channel signal $\mathbf H \in \mathbb R^{n \times d_0}$ as 

\begin{align}
    &\mathcal W_{0,J} = \mathbf U \widehat{a}(\frac{\boldsymbol{\Lambda}}{2^{m+J}}) \cdots \widehat{a}(\frac{\boldsymbol{\Lambda}}{2^{m}}) \mathbf U^\top \label{eq_w0j}\\
    &\mathcal W_{r,0} = \mathbf U \widehat{b^{(r)}}(\frac{\boldsymbol{\Lambda}}{2^{m}}) \mathbf U^\top, \;\;\text{for } r = 1,...,L, \label{eq_wr0}\\
    &\mathcal{W}_{r,j} = \mathbf U \widehat{b^{(r)}}(\frac{\boldsymbol{\Lambda}}{2^{m+j}})\widehat{a}(\frac{\boldsymbol{\Lambda}}{2^{m+j-1}}) \cdots \widehat{a}(\frac{\boldsymbol{\Lambda}}{2^{m}}) \mathbf U^\top, \label{eq_wrj}\\ 
    \;\;\;\; & \text{for }  r = 1,...,L, j = 1,...,J.\notag
\end{align}
where $m$ is the coarsest scale level which is the smallest value satisfying $2^{-m}\lambda_n \leq \pi$. It can be shown that $\sum_{(r,j) \in \mathcal I} \mathcal W_{r,j}^\top \mathcal W_{r,j} = \mathbf I$ for $\mathcal I = \{(r,j) : r = 1,...,L, j = 0, 1,...,J \} \cup \{ (0, J) \}$, indicating the \textit{tightness} of framelet decomposition and reconstruction.  

To mitigate potential computational cost imposed by the eigendecomposition of graph Laplacian, Chebyshev polynomials are a commonly utilized to approximate the framelet transformation matrices.
Empirically, the implementation of Chebyshev polynomials $\mathcal{T}^n_j(\xi)$ with a fixed degree $n$ has shown great power of approximating the corresponding images of the filtering function such as  $\widehat{b}(\xi)$. Thus, one can define so-called Quasi-Framelet transformation matrices \cite{yang2022quasi} for Eq.~\eqref{eq_w0j} - Eq.~\eqref{eq_wrj} as:
\begin{align}
    \mathcal{W}_{0,J} &\approx 
    \mathcal{T}_0(\frac1{2^{m+j}}\widehat{\mathbf L}) \cdots \mathcal{T}_0(\frac{1}{2^{m}}\widehat{\mathbf L}),   \label{eq:Ta0}\\
    \mathcal{W}_{r,0} &\approx  
    \mathcal{T}_r(\frac1{2^{m}}\widehat{\mathbf L}),  \;\;\; \text{for } r = 1, ..., L, \label{eq:Ta}
 \\ 
    \mathcal{W}_{r,j} &\approx 
    \mathcal{T}_r(\frac{1}{2^{m+j}}\widehat{\mathbf L})\mathcal{T}_0(\frac{1}{2^{m+j-1}}\widehat{\mathbf L}) \cdots \mathcal{T}_0(\frac{1}{2^{m}}\widehat{\mathbf L}), \label{eq:Tc}\\
    &\;\;\;\; \text{for } r=1, ..., L, j= 1, ..., J. \notag
\end{align}

Based on the theoretical and empirical support from  framelet transformation, many attempts have been made by applying  graph framelet on graph learning tasks, which can be divided into two camps. The first is \textit{spectral framelet} proposed in \cite{yang2022quasi,YangZhouYinGao2022}. 
The node feature propagation rule of spectral framelet from layer $\ell$ to $\ell+1$ is expressed as: 
\begin{align}
\mathbf H^{(\ell + 1)} = &\sigma\left( \sum_{(r,j)\in\mathcal I}\mathcal W_{r,j}^\top {\rm diag}(\mathbf \theta_{r,j}) \mathcal W_{r,j} \mathbf H^{(\ell)} \mathbf W^{(\ell)}\right) 
\end{align}  
where ${\rm diag}(\theta) \in \mathbb R^{n\times  n}$ contains learnable coefficients in each frequency domain and $\mathbf W^{(\ell)}$ is a shared weight matrix. Rather than performing spectral filtering, {\it spatial framelet} \cite{chendirichlet} convolution propagates the node information in  spatial domain with the following propagation rule (without activation):

\begin{align}\label{spatial_framelt_teacher}
    &\mathbf H^{(\ell +1)}  \!\!= \sum_{(r,j)\in \mathcal I} \mathcal W^\top_{r,j} \widehat{\mathbf A} \mathcal W_{r,j} \mathbf H^{(\ell)} \mathbf W^{(\ell)}_{r,j}.
\end{align}
In the rest of this paper, our analysis will mainly focus on spatial framelet in Eq.~\eqref{spatial_framelt_teacher}, simply referred to as framelet,
although we will also present some conclusions regarding spectral framelet in the later sections.

The GNNs simplification is first introduced in \cite{wu2019simplifying}, and the core idea is to reveal the fact that the hidden representation produced in each layer of GCN is averaged among the neighbours that are one hop away. This suggests after $\ell$ layers, graph nodes will aggregate the feature information from the nodes that are $\ell$-hops away in the graph, and such $\ell$-hop information is contained in the $\ell$-th power of the adjacency matrix (i.e., $\widehat{\mathbf{A}}^\ell$). Thus one can simplify the GCN propagation presented in Eq.~\eqref{eq_classic_gcn} with layer $\ell$ as: 
\begin{align}\label{simplified_gnn}
    \mathbf H^{(\ell + 1)} = \sigma \big( \widehat{\mathbf A}^\ell \mathbf H^{(0)} \mathbf W \big),
\end{align} 
where matrix $\mathbf W \in \mathbb R^{d_0 \times d_{\ell+1}}$. 
When the network directly transfers the node features to the labels, then $d_{\ell+1}=c$ is the dimension of the total number of classes. 
Thus, followed by the spirit of GNNs simplification, we show below the (linearized) simplified spatial framelet 
as:     
\begin{align}\label{simplified_framelet}
     \mathbf H^{(\ell + 1)}  =  
  \sum_{(r,j)\in\mathcal{I}} \left(\mathcal W_{r,j}^\top  \widehat{\mathbf A} \mathcal W_{r,j} \right)^\ell \mathbf H^{(0)} \mathbf W_{r,j},  
\end{align}
where 
all $\mathbf W_{r,j}$ are of size $d_0 \times c$. There is a significant difference between the simplified version Eq.~\eqref{simplified_framelet} and the application of $\ell+1$ layers of Eq.~\eqref{spatial_framelt_teacher}. In Eq.~\eqref{simplified_framelet}, we repeatedly apply the filter at the same scales, no crossing scale application.  In this paper, we limit $\ell$ up to $2$ for convenience in most cases. 

As we mentioned in the section for knowledge distillation, one of the keys of a successful knowledge distillation is to supply student model with sufficient graph information provided by the teacher model. We will show how a {MLP} based student model can get this task done, i.e., learning from the teacher model as Eq.~\eqref{simplified_framelet} and Eq.~\eqref{spatial_framelt_teacher} in the coming section.

\section{The Proposed Method}\label{propose_model}
The core idea of our method is to make student model sufficiently mimic teacher model, i.e., Eq.~\eqref{simplified_framelet} and Eq.~\eqref{spatial_framelt_teacher}. To do this, in this section, we first take the original framelet model Eq.~\eqref{spatial_framelt_teacher} as the teacher model example to propose a $\mathrm{MLP}$ based encoding and decoding scheme as a student model. 
Then we show how this student model imitates simplified framelet Eq.~\eqref{simplified_framelet}. Furthermore, we provide some discussions, including the relationship between these two models and their potential computational issues. Since this $\mathrm{MLP}$ based student model tends to unify graph knowledge, including both adjacency and node features, resulting a ``frameless''information propagation, we name our framelet distillation models as frameless multilayer perceptron original (FLMP-O) and frameless multilayer perceptron simplified (FMLP-S).



\subsection{FMLP-O} 
\subsubsection{Encoding Graph Knowledge}\label{section_model_encoding}
As we have illustrated before, we first take the original spatial framelet as the teacher model to show how FLMP-O is built. To encode graph knowledge, first we see that the framelet propagation rule in  Eq.~\eqref{spatial_framelt_teacher} is a conditional message-pass scheme in  framelet  domain. Therefore, the knowledge provided by the framelet model at the first layer ($\ell =1$) is: $\mathcal W^\top_{0,J} \widehat{\mathbf A}\mathcal W_{0,J} $, $\mathbf X$ and $\mathcal W^\top_{r,j} \widehat{\mathbf A}\mathcal W_{r,j}$ ($r=1, ..., L,j= 1, ..., J$). We directly feed this information to {MLP} to so that FLMP-O can capture the graph information. We note that this approach is parallel to any encoding process such as those in \cite{joulin2016fasttext,wang2014generalized,ke2020rethinking} that involves {MLP} as the encoder without any additional constraints. Specifically, we have: 
\begin{align}
    \mathbf Q_{0,J}^{(1)} &= \text{MLP}^{(1)}_{0,J}(\mathcal W^\top_{0,J} \widehat{\mathbf A} \mathcal W_{0,J}), \,\,\,\,\,\,
    \mathbf H^{(1)}_\mathbf X  = \text{MLP}^{(1)}_{\mathbf X}(\mathbf X)\,\,\,\,\,\,  \notag \\
    \mathbf Q_{r,j}^{(1)} & = \text{MLP}^{(1)}_{r,j}(\mathcal W^\top_{r,J} \widehat{\mathbf A}\mathcal W_{r,J}), \,\,\,\,\,\, r = 1, ..., L, j= 1, ..., J\label{first_encoding}
\end{align}
where the superscription $^{(1)}$ stands for the encoding process that targets on the first layer information of framelet and the encoded matrices $\mathbf Q_{0,J}^{(1)}$, $\mathbf Q_{r,j}^{(1)}$ and $\mathbf H^{(1)}$ are of size $n \times d_0$. It has been shown that $\mathrm{MLP}$ is a universal approximator in $L^p$ space and thus has the power to approximate any function defined in it with arbitrary precision \cite{arora2016understanding,royden1988real}. This is why it is chosen for encoding and decoding the graph knowledge. In addition, the computational benefits of involving {MLP} are obvious. For example, one may engage the training on mini-batch if the input graph is large. Furthermore,  the inclusion of {MLP} allows the student model to drop the relatively complicated graph information propagation process, i.e., the model designed in \cite{chen2020self} and \cite{yang2020distilling}, instead, the whole encoding process, including the decoding process presented in later sections, is done only by the matrix product individually to the graph information. Thus the model's architectural complexity is largely simplified.

\begin{figure*}
    \includegraphics[scale=0.35]{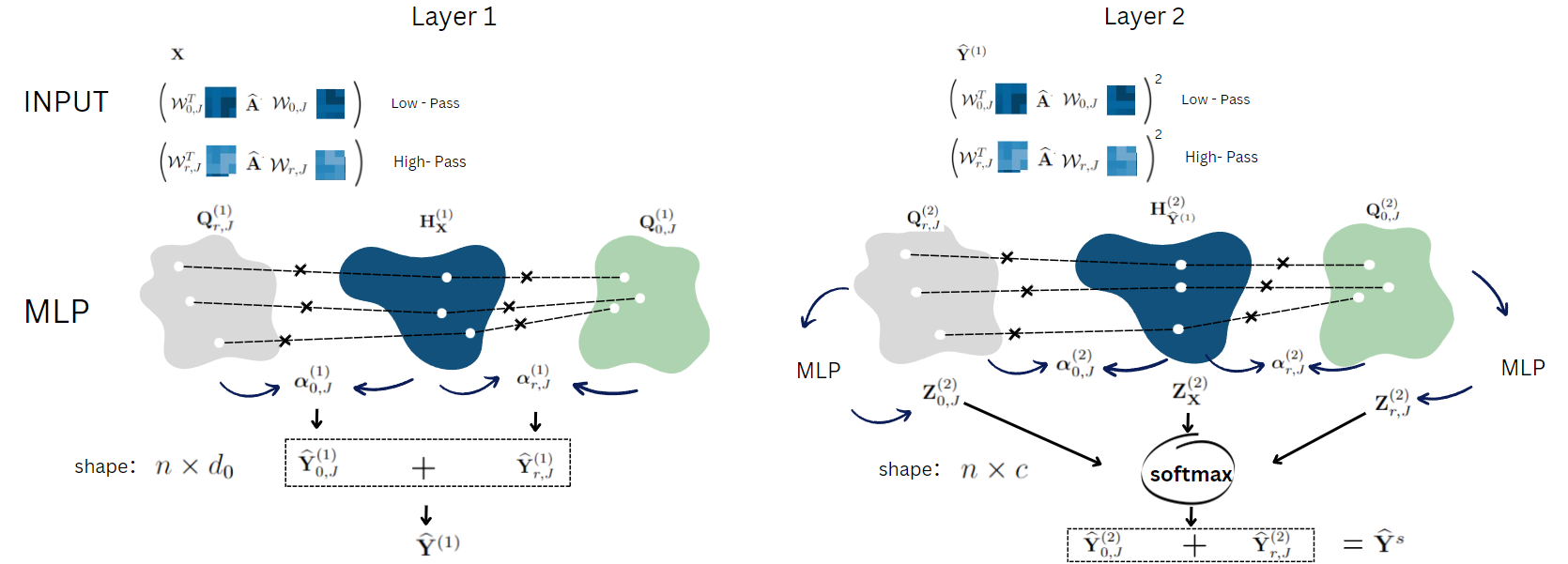}
    \caption{The figure above illustrates the workflow of FMLP-O, one of our student models. The model first imitates the learning  process of spatial framelet teacher model \ref{spatial_framelt_teacher} with only one low and high-pass filter by encoding the graph feature knowledge ($\mathbf X)$ and framelet decomposed adjacency knowledge, i.e., $\mathcal W_{r,j}^\top \widehat{\mathbf A}\mathcal W_{r,j}$. using $\mathrm{MLP}$. The relative importance of the encoded graph knowledge is then captured by the score vectors $\boldsymbol{\alpha}$ to fit both homophily and heterophily graph inputs in the encoding space. With the help of score vectors, the graph knowledge is processed.  The output of the first layer of FMLP-O ($\widehat{\mathbf Y}^{(1)})$ is then utilized as the graph feature knowledge together with the desired adjacency information, i.e. $(\mathcal W^\top_{r,j} \widehat{\mathbf A}\mathcal W_{r,j})^{2}$, to enable FMLP-O to mimic the second layer which outputs the final prediction of the student model. }
    \label{fig:FMLPO}
\end{figure*}

\subsubsection{Score-based Decoding}\label{section_decoding}
\paragraph{First-round decoding}\label{first_round_deconding}
Before we introduce any score-based decoding process, we point out that there is a unsolved question left in the previous section, that is: 
\textit{How can we ``digest'' the information generated from Eq. ~\eqref{first_encoding} and prepare it for the second layer?} 
To address this problem, an intermediate outcome based on the result of  Eq.~\eqref{first_encoding} has to be generated. Followed by the framelet reconstruction process, which sums the outcomes of both low-pass and high-pass predictions, we define score vectors: $\boldsymbol{\alpha}^{(1)}_{0,J}$ and $\boldsymbol{\alpha}^{(1)}_{r,j}$, $r = 1, ..., L,j= 1, ..., J$ 
all are of size $n$ generated from the following process: 
\begin{align}\label{first_score_vector}
    &\boldsymbol{\alpha}^{(1)}_{0,J} = \mathrm{sigmoid} ([\mathbf Q_{0,J}^{(1)}||\mathbf H^{(1)}_\mathbf X]\mathbf P^{(1)}_{0,J}+ b^{(1)}_{0,J}), \,\,\,  \notag \\ &\boldsymbol{\alpha}^{(1)}_{r,j} = \mathrm{sigmoid}([\mathbf Q_{r,j}^{(1)}||\mathbf H^{(1)}_\mathbf X]\mathbf P^{(1)}_{r,j}+ b^{(1)}_{r,j}),
\end{align}
where $\mathbf P^{(1)}_{0,J},\mathbf P^{(1)}_{r,j} \in \mathbb R^{2d_0 \times 1}$, $b^{(1)}_{0,J},b^{(1)}_{r,j} \in \mathbb R$ 
are trainable parameters and $[\cdot || \cdot]$ is the matrix concatenation. Therefore the outcome of the distillation for the first layer of the framelet can be computed as the summation of the student low-pass and high-pass sessions: 
\begin{align}\label{first_layer_decoding}
    &\widehat{\mathbf Y}^{(1)}_{0,J} =\boldsymbol{\alpha}^{(1)}_{0,J}\cdot\mathbf  H^{(1)}_\mathbf X + (1- \boldsymbol{\alpha}^{(1)}_{0,J}) \cdot\mathbf Q_{0,J}^{(1)}, \notag \\
    &  \widehat{\mathbf Y}^{(1)}_{r,j} =  \boldsymbol{\alpha}^{(1)}_{r,j}\cdot\mathbf H^{(1)}_\mathbf X + (1-\boldsymbol{\alpha}^{(1)}_{r,j}) \cdot\mathbf Q_{r,j}^{(1)}.   
\end{align}
We then take $\widehat{\mathbf Y}^{(1)} = \widehat{\mathbf Y}^{(1)}_{0,J}+ \sum_{r,j} \widehat{\mathbf Y}^{(1)}_{r,j}$ for the reconstruction process in the first decoding process in the student model. We point out that there are many other score generating schemes in the  literature such as \cite{song2020score,vahdat2021score} that focus on different learning tasks. The interpretation of ours is that the output of $\mathrm{sigmoid}$ function balances between the encoded feature and the graph information in  framelet domain. 
Furthermore, they connect to label consistency on graph 
\cite{li2018deeper} that is quantified by the so-called homophily index \cite{chien2020adaptive}. Homophily index indicates the uniformity of labels  in neighborhoods. A higher homophily index corresponds to the graph where more neighborhoods members share the same label. \guorm{presents the average proportion of whether two nodes with the same class label are connected with each other for any given graph.} When the graph is homophilic, both the encoded adjacency and feature information are all important to generate the node representations. In contrast, when the graph is heterophilic, the adjacency information should be less important after the propagation process, as a sharpening rather than smoothing effect  in terms of label prediction \cite{di2022graph} is preferred. Thus one may prefer to have a sparser (less connection) graph or relatively weak connection (smaller edge weights) between nodes \cite{rong2019dropedge}. 
 In this case, Eq.~\eqref{first_layer_decoding} tends to generate a relatively larger $\boldsymbol{\alpha}$ so that encoded node feature information dominates the learning outcome.

\paragraph{Second-round decoding}
Once we have obtained the outcome $\widehat{\mathbf Y}^{(1)}$ from the first round decoding, we set $\widehat{\mathbf Y}^{(1)}$ as the input node features to the second round encoding process of FLMP. Based on  Eq.~\eqref{first_layer_decoding}, $\widehat{\mathbf Y}^{(1)}$ is of size $n \times d_0$. According to the propagation scheme in framelet ($\ell =2$) in  
Eq.~\eqref{spatial_framelt_teacher}, the graph knowledge should rely on $(\mathcal W^\top_{0,J} \widehat{\mathbf A}\mathcal W_{0,J})^{2}$ and $(\mathcal W^\top_{r,j} \widehat{\mathbf A}\mathcal W_{r,j})^{2}$. Therefore, we have the following encoding process:
\begin{align}\label{second_encoding}
    \mathbf Q_{0,J}^{(2)} &= \text{MLP}^{(2)}_{0,J}\left((\mathcal W^\top_{0,J} \widehat{\mathbf A}\mathcal W_{0,J})^{2}\right), \;\mathbf H^{(2)}_{\widehat{\mathbf Y}^{(1)}}  = \text{MLP}^{(2)}_{\widehat{\mathbf Y}^{(1)}}(\widehat{\mathbf Y}^{(1)})\notag \\
    \mathbf Q_{r,j}^{(2)} &= \text{MLP}^{(2)}_{r,j}\left((\mathcal W^\top_{r,j} \widehat{\mathbf A}\mathcal W_{r,j})^{2}\right), \!\;\;\; r = 1, ..., L,j= 1, ..., J.
\end{align}
Here we emphasize that the superscription $^{(i)}$ stands for the layer index and the case without brackets $i$ means the matrix exponential as in  Eq.~\eqref{spatial_framelt_teacher}. Note that it is optional to encode  $\widehat{\mathbf Y}^{(1)}$ since it is initially generated from the first-round encoding process, although we still encode $\widehat{\mathbf Y}^{(1)}$ for clarity. Thus, the score vectors for the encoded knowledge from above are:
\begin{align}\label{second_score_vector}
    &\boldsymbol{\alpha}^{(2)}_{0,J} = \mathrm{Sigmoid} ([\mathbf Q_{0,J}^{(2)}||\mathbf H^{(2)}_{\widehat{\mathbf Y}^{(1)}}]\mathbf P^{(2)}_{0,J}+ b^{(2)}_{0,J}), \notag \\ &\boldsymbol{\alpha}^{(2)}_{r,j} = \mathrm{Sigmoid}([\mathbf Q_{r,j}^{(2)}||\mathbf H^{(2)}_{\widehat{\mathbf Y}^{(1)}}]\mathbf P^{(2)}_{r,j}+ b^{(2)}_{r,j}),
\end{align}
Different from the first round encoding and decoding process, we shall drive the student model  to match the teacher model after the second decoding process. After the score vectors generated in  Eq.~\eqref{second_score_vector}, we decode all the information encoded from the second round encoding  Eq.~\eqref{second_encoding} to $c$ dimensional space, where $c$ is the number of classes. This is implemented by the following:
\begin{align}\label{label_dimension_encoding}
&\mathbf Z^{(2)}_{0,J} =  \text{MLP}^{(2)}_{\mathbf Z_{0,J}}\left(\mathbf Q_{0,J}^{(2)}\right), \notag \\ 
&\mathbf Z_{\mathbf X}^{(2)} = \text{MLP}^{(2)}\left(\mathbf H^{(2)}_{\widehat{\mathbf Y}^{(1)}}\right), 
\mathbf Z^{(2)}_{r,j} = \text{MLP}^{(2)}_{\mathbf Z_{r,j}}\left(\mathbf Q_{r,j}^{(2)}\right),
\end{align}
where 
$\mathbf Z_{\mathbf X}^{(2)}$ is generated by inputting the node presentation from the first decoding to a $\text{MLP}$ named as $\text{MLP}^{(2)}$. Thus, we have both $\mathbf Z^{(2)}_{0,J}$, $\mathbf Z^{(2)}_{r,j}$ and $\mathbf Z_{\mathbf X}^{(2)}$ of size $n \times c$. Analogously, we assign both $\boldsymbol{\alpha}^{(2)}_{0,J}$ and $\boldsymbol{\alpha}^{(2)}_{r,j}$ to balance the importance between decoded node feature in $\mathbf Z_{\mathbf X}^{(2)}$ and the framelet based graph connectivity in $\mathbf Z^{(2)}_{0,J}$ and $\mathbf Z^{(2)}_{r,j}$, and generate the final output of the FLMP, that is: 
\begin{align}\label{second_round_final}
  &\widehat{\mathbf Y}^{(2)}_{0,J} =  \boldsymbol{\alpha}^{(2)}_{0,J}\cdot \mathbf Z^{(2)}_\mathbf X + (1- \boldsymbol{\alpha}^{(2)}_{0,J}) \cdot \mathbf Z_{0,J}^{(2)} , \,\,\, 
    \,\,\, \notag  \\ 
    &\widehat{\mathbf Y}^{(2)}_{r,j} = \boldsymbol{\alpha}^{(2)}_{r,j}\cdot \mathbf Z^{(2)}_\mathbf X + (1-\boldsymbol{\alpha}^{(2)}_{r,j})\cdot  \mathbf Z_{r,j}^2, \notag \\
    & \widehat{\mathbf Y}^s = \mathrm{Softmax}\left(\widehat{\mathbf Y}^{(2)}_{0,J}+ \sum_{(r,j)\in\mathcal{I}} \widehat{\mathbf Y}^{(2)}_{r,j}\right),
\end{align}
where the superscription $^s$ stands for the variable associated with the student model. Note that it is necessary to generate second score vectors  $\boldsymbol{\alpha}^{(2)}_{0,J}$ and $\boldsymbol{\alpha}^{(2)}_{r,j}$ as their functionality is to balance the importance between encoded nodes features and the graph information that utilized in the second layer of framelet. The model is trained by minimizing a loss function with two components. One is the cross-entropy loss, denoted as $\mathcal L_{\mathrm{Label}}$, between  true labels $\mathbf Y$ and the student prediction $\widehat{\mathbf Y}^s$. The other is   the KL-divergence, denoted as $\mathrm{KL}$, between student and teacher's predictions \cite{zhang2021graph}. The complete cost function is: 
\begin{align}\label{training_function}
    \mathcal L = \lambda \mathcal L_{\mathrm{Label}}\left(\widehat{\mathbf Y}^s, \mathbf Y\right) + (1-\lambda) \mathrm{KL}\left(\widehat{\mathbf Y}^\top, \widehat{\mathbf Y}^s\right),
\end{align}
where $\lambda\in[0,1]$ is a hyperparameter that balances the relative importance between two cost components. 
To summarize, \guorm{the model structure of} FMLP-O \guorm{is built by gathering} initiates the first round encoding and decoding process (Eq.~\eqref{first_encoding}-Eq.~\eqref{first_layer_decoding}) to distill graph knowledge and imitate the learning process of the first layer of spatial framelet; with its learning outcome generated from the reconstruction in $\widehat{\mathbf Y}^{(1)}$, FMLP-O conducts the second round encoding and decoding for prediction and imitation  through Eq.~\eqref{second_encoding} -Eq.~\eqref{training_function}. 
Fig.~\ref{fig:FMLPO} shows the learning process of FMLP-O.

\subsection{FLMP-S and Discussions}\label{FLMPS_and_discussions}
Simplified framelet Eq.~\eqref{simplified_framelet} directly propagates the graph information utilized in the last layer \guorm{(i.e., $\ell =2$)} of framelet together with the original graph signals to generate the final outcome. Therefore, in terms of the design of FLMP-S, one can utilize the process of second-round encoding and decoding from Eq.~\eqref{second_encoding}- Eq.~\eqref{second_round_final}  with signal reconstruction\guorm{ to generate the learning outcome}. The only difference is that, in FLMP-O, the input feature matrix for Eq.~\eqref{second_encoding} is the reconstructed feature matrix $\widehat{\mathbf Y}^{(1)}$ from the first encoding,  whereas, in FLMP-S, we encode the original feature matrix $\mathbf X$ directly and utilize it with the graph information that propagated in the last layer of the original framelet, i.e., $(\mathcal W^\top_{0,J} \widehat{\mathbf A} \mathcal W_{0,J})^\ell$ and $(\mathcal W^\top_{r,j} \widehat{\mathbf A} \mathcal W_{r,j})^\ell$. 


\subsection{Dirichlet Energy for Teacher Models}
It has been shown that a higher energy of node representation is preferred when the graph is heterophilic. However, due to the propagation difference between teacher models, their adaption power to  heterophilic graph tends to be different. In this section, we discuss the difference between the proposed teacher models from energy perspective. We first define the graph Dirichlet energy. 
\begin{defn}[Graph Dirichlet Energy]
Given node embedding matrix $\mathbf H^{(\ell)}\in \mathbb R^{n\times d_\ell}$ learned from GNN at layer $\ell$, 
the Dirichlet energy $E(\mathbf H^{(\ell)})$ is defined as: 
\begin{align*}
    E(\mathbf H^{(\ell)}\!)\! &=  \frac{1}{2} \trace((\mathbf H^{(\ell)})^\top \widehat{\mathbf L} \mathbf H^{(\ell)}) 
    = \frac{1}{4} \sum_{i,j\in\mathcal{V}} a_{ij} \| \mathbf h_i^{(\ell)} /\sqrt{d_i} - \mathbf h_j^{(\ell)} /\sqrt{d_j} \|^2 \notag,
\end{align*}
where $d_i$ and $d_j$ are the degrees of node $i$ and $j$. 
\end{defn}
Based on Proposition 1 in \cite{chendirichlet}, the total amount of graph Dirichlet energy in one specific layer \guorm{(i.e., $\ell$)} is conserved under framelet decomposition.\guorm{ (applied on the node representation at that layer). Specifically,} For node representation 
$\mathbf H^{(\ell)} \in \mathbf R^{n\times d_\ell}$, let 
\[
E_{0,J}(\mathbf H^{(\ell)}) = \frac{1}{2} \trace \big( (\mathcal W_{0,J} \mathbf H^{(\ell)})^\top \widehat{\mathbf L} \mathcal W_{0,J} \mathbf H^{(\ell)} \big),
\]
and similarly 
\[
E_{r,j}(\mathbf H^{(\ell)}) = \frac{1}{2} \trace \big( (\mathcal W_{r,j} \mathbf H^{(\ell)})^\top \widehat{\mathbf L} \mathcal W_{r,j} \mathbf H^{(\ell)} \big)
\]
for $r=1,...,L, j = 0, 1,...,J$. Then we have 
\[
E(\mathbf H) = 
\sum_{(r,j)\in\mathcal{I}} E_{r,j}(\mathbf H^{(\ell)}),
\]
where $E_{0,J}$ and $E_{r,j}$ are the Dirichlet energy in low and high frequency domains of framelets. In \cite{chendirichlet}, an energy perturbation scheme is proposed to increase total Dirichlet energy for one specific layer in Eq. ~\eqref{spatial_framelt_teacher} by slightly decreasing the energy from low-pass domain while increasing the energy from high-pass domain. That is, for a positive constant $\epsilon$ we have:  

\begin{align}\label{e:energypert}
    E^{\epsilon}(\mathbf H) &= \frac{1}{2} \trace  \big( (\mathcal W_{0,J} \mathbf H)^\top  (\widehat{\mathbf L} + \epsilon \mathbf I_n) \mathcal W_{0,J} \mathbf H  \big) + 
    \frac{1}{2} \sum_{(r,j)\in\mathcal{I} \backslash (0, J)} \trace \big( (\mathcal W_{r,j} \mathbf H)^\top (\widehat{\mathbf L} - \epsilon \mathbf I_n ) \mathcal W_{r,j} \mathbf H \big). 
\end{align} 
If the Haar-type filter is applied in spatial framelet Eq.~\eqref{spatial_framelt_teacher}, one can show that $E^\epsilon(\mathbf H) > E(\mathbf H)$ for any $\epsilon > 0$ due to the energy gap between low-pass and high-pass filters\guorm{ in this case}.
Now we verify that this energy conservation property may no longer be held in simplified framelet in~Eq.~\eqref{simplified_framelet}. Instead, it depends on the eigen-distribution of the normalized graph Laplacian $\widehat{\mathbf L}$. For a given layer $\ell$, one can simply define  low-pass and high-pass Dirichlet energy under the decomposition induced from simplified framelet analogously as:  
\[
E_{0,J}(\mathbf H^{(\ell)}) = \frac{1}{2} \trace \big( \mathcal W_{0,J}^{\ell} \mathbf H^{(\ell)})^\top \widehat{\mathbf L}( \mathcal W_{0,J}^{\ell} \mathbf H^{(\ell)} \big),
\]
and  
\[
E_{r,j}(\mathbf H^{(\ell)}) = \frac{1}{2} \trace \big( (\mathcal W_{r,j}^{\ell} \mathbf H^{(\ell)})^\top \widehat{\mathbf L} (\mathcal W_{r,j}^{\ell} \mathbf H^{(\ell)} \big)
\]for $r=1,...,L, j = 0, 1,...,J$, since the decomposition induced from simplified framelet at layer $\ell$ is $\mathcal W_{0,J}^{\ell}$ and $\mathcal W_{r,j}^{\ell}$. Let $E^\mathcal S$  be the total Dirichlet energy in simplified framelet  at layer $\ell$. To avoid cluttered notations, 
we use $\mathbf H$ instead of $\mathbf H^{(\ell)}$ in the following explanations. 
Referring to Eq.~\eqref{eq_w0j} -- Eq.~\eqref{eq_wrj} for the meaning of $\Lambda_{r,j}$, we have 
\begin{align}
    &E^\mathcal S (\mathbf H) = \!\! \mathrm{Tr}\sum_{(r,j)\in\mathcal{I}} \left((\mathcal W_{r,j}^\top)^{\ell} \mathbf H \right)^\top \!\! \widehat{\mathbf L}\left(\mathcal W_{r,j}^\ell \mathbf H   \right) \notag \\
    &= \mathrm{Tr}\left(\sum_{(r,j)\in\mathcal{I}}  \mathbf H^\top \left(\mathbf U^\top \mathbf  \Lambda_{r,j} \mathbf {U}\right)^\ell 
    \mathbf U^\top \mathbf \Lambda \mathbf U\left(\mathbf U^\top \mathbf  \Lambda_{r,j} \mathbf {U}\right)^\ell\mathbf H \right) \notag \\
    & = \mathrm{Tr} \left(\sum_{(r,j)\in\mathcal{I}} (\mathbf U \mathbf H)^\top \mathbf  \Lambda_{r,j}^{2\ell}  \mathbf \Lambda (\mathbf {U}\mathbf H) \right) \label{Eq:22}
\end{align}

Let us consider $\widehat{\mathbf H} = \mathbf U\mathbf H$, i.e., Fourier framelet signals, and denote its columns as $\widehat{\mathbf h}_c$ ($c=1, ..., d_{\ell}$). Then Eq.~\eqref{Eq:22} leads to
\begin{align*}
E^\mathcal S (\mathbf H) &= \sum^{d_{\ell}}_{c=1} \sum_{(r,j)\in\mathcal{I}} \sum_i    \widehat{\mathbf h}^2_c(i)  (\lambda^i_{r,j})^{2\ell} \lambda_i  
= \sum^{d_{\ell}}_{c=1}\sum_i  \lambda_i \widehat{\mathbf h}^2_c(i) \sum_{(r,j)\in\mathcal{I}}  (\lambda^i_{r,j})^{2}(\lambda^i_{r,j})^{2\ell-2}.  
\end{align*}
Clearly, when $\ell =1$ we can recover the total Dirichlet energy proposed in \cite{chendirichlet}. This comes from the fact that $(\lambda^i_{r,j})^{2\ell-2} = 1$ and $\sum_{(r,j)\in\mathcal{I}}  (\lambda^i_{r,j})^{2} = 1$ for all $i$  due to the tightness of framelet decomposition, a.k.a, $\sum_{r,j} \mathbf \Lambda_{r,j}^2 = \mathbf I_n$.  Thus the above equation can be further simplified to
\[
\sum^{d_{\ell}}_{c=1}\sum_i  \lambda_i \widehat{\mathbf h}^2_c(i) = \mathrm{Tr}\left(\mathbf H^\top \mathbf U^\top\mathbf  \Lambda \mathbf U \mathbf H\right) = E(\mathbf H).
\]However, when $\ell > 1$, whether the energy is conserved depends on the distribution of eigenvalues of $\widehat{\mathbf L}$. Hence to increase  Dirichlet energy, one may resort to $\epsilon$ in Eq.~\eqref{e:energypert} with different signs, 
since the approach presented in \cite{chendirichlet} may not always be able to lift the graph Dirichlet energy in simplified framelet.
Nevertheless, we still observe in our experiments  that by assigning a proper quantity of energy perturbation (denoted as $\boldsymbol{\epsilon}^S$ for simplified framelet), the simplified framelet still achieves a remarkable outcome in heterophily graphs. 

\subsection{Alleviate Over-Squashing}  
Recall that we set $\ell$ up to $2$ in most of practical implementations. Here we discuss a potential issue when $\ell$ becomes large, say $\ell\gg2$ so that the information smears out in propagation. It is known as over-squashing. The over-squashing issue is a recently identified problem within GNN learning process. An intuitive description of the over-squashing issue is that the long-term (multi-hop) connectivity importance (dependence) between nodes tends to be diluted as the number of layers of a GNNs model becomes large. In \cite{topping2021understanding}, the over-squashing phenomenon was quantified by the norm of the Jacobian, also known as sensitivity, between the node representations obtained by a GNN model in any one layer and  the original data. One can formulate it node-wise, i.e. the sensitivity between node $v$ at layer $\ell$ to the node $u$ in $\mathbf X$, which is $\mathbf H^{(0)}$, is $\left\|\frac{\partial \mathbf h^{(\ell)}_v}{\partial \mathbf x_u} \right\|_2$. Here we show a more detailed version of sensitivity by measuring the quantity of $\left |\frac{d \mathbf h^{(\ell)}_{v,q}}{d \mathbf x_{u,p}} \right|$ which is the sensitivity between one specific element of feature $\mathbf h_v$ at layer $\ell$ and one element of $\mathbf x_u$ at layer 0 for both of our teacher models. The following lemma shows that such quantity critically depends on the  weights of the graph adjacency matrix. Without loss of generality, 
we start with the spatial framelet model Eq.~\eqref{spatial_framelt_teacher}, and later generalize for simplified framelet Eq.~\eqref{simplified_framelet}. 

\begin{lem}\label{spatial_framelet_over_squash}
Consider the propagation in spatial framelet Eq.~\eqref{spatial_framelt_teacher}: $\mathbf H^{(\ell +1)}  = 
\sum_{(r,j)\in\mathcal{I}} \mathcal W_{r,j}^\top  \widehat{\mathbf A} \mathcal W_{r,j} \mathbf H^{(\ell)} \mathbf W^{(\ell)}_{r,j}$ at layer $\ell +1$ with  $\mathbf H^{(0)} = \mathbf X$. Let  
$\widehat{\mathbf A}_{r,j} = \mathcal W_{r,j}^\top \widehat{\mathbf{A}} \mathcal W_{r,j}$ be the graph adjacency information in various frequency domains, 
$\mathbf h_{v,q}^{(\ell+1)}, q \in [1, d_{\ell+1}]$, the $q$-th component of $v$-node feature at layer $\ell+1$, and $\mathbf W^{(\ell)}_{r,j}(p,q)$  the $(p,q)$-th entry of $\mathbf W^{(\ell)}_{r,j}$.  Suppose $|\mathbf W^{(\ell)}_{r,j}(p,q)| \leq \beta_{r,j}$ for all $\ell$. Then for any pair of nodes $u$, $v$ 
which are in the standard shortest path that equals to $\ell+1$, we have $\left |\frac{d \mathbf h^{(\ell+1)}_{v,q}}{d \mathbf x_{u,p}} \right| \leq (d\beta)^{\ell+1}\left(\sum_{(r,j)\in\mathcal{I}} \left|\widehat{\mathbf A}_{r,j}\right|\right)^{\ell+1}_{u,v} $, where $d = \max\{d_0, d_1, ..., d_{\ell+1}\}$ and $\beta = \max\{ \beta_{r,j}: (r,j)\in\mathcal{I}\}$.
\end{lem}
\begin{proof}
First, we can simplify the propagation of spatial framelet in Eq.~\eqref{spatial_framelt_teacher} as $\mathbf H^{(\ell+1)} = \sum_{(r,j) \in \mathcal I}\widehat{\mathbf A}_{r,j} \mathbf H^{(\ell)} \mathbf W_{r,j}^{(\ell)}$. Then we have: 
\begin{align*}
    &\mathbf h_{v,q}^{(\ell+1)} = \sum_{(r_i,j_1)\in \mathcal I }\sum_{w_1}\sum_{p_1} a_{r,j_1}(v,w_1)\mathbf h^{(\ell)}_{w_1,p_1}\cdot \mathbf W_{r_i,j_1}^{(\ell)}(p_1,q) \\
    & =  \sum_{(r_1,j_1)\in \mathcal I }\sum_{w_1,p_1}a_{r,j_1}(v,w_1) 
    \left(\sum_{(r_2,j_2)\in \mathcal I } \sum_{w_2,p_2} 
    a_{r_2,j_2}(w_1,w_2)\mathbf h_{w_2,p_2}^{(\ell-1)}\mathbf W^{(\ell-1)}_{r_2,j_2}(p_2,p_1) \big)\mathbf W_{r_i,j_1}^{(\ell)}(p_1,q)\right) \\
    & = \sum_{(r_1,j_1)\in \mathcal I}\sum_{(r_2,j_2)\in \mathcal I}\cdots\sum_{(r_{\ell+1)},j_{\ell+1})\in \mathcal I}\sum_{w_{\ell+1}, p_{\ell+1}} \\
    & a_{r_1,j_1}(v,w_1)a_{r_2,j_2}(w_1,w_2)\cdots a_{r_{\ell+1},j_\ell}\mathbf h^{(0)}_{w_{\ell+1},p_{\ell+1}} 
    \mathbf W^{(0)}_{r_{\ell+1},j_{\ell+1}}(p_{\ell+1}, q_\ell)\mathbf W^{(1)}_{r_{\ell},j_{\ell}}(p_{\ell}, q_{\ell-1})\cdots \mathbf W^{(\ell)}_{r_{1},j_{1}}(p_{1}, q).
\end{align*}
Hence
\begin{align*}
 &\frac{d\mathbf h_{v,q}^{(\ell+1)}}{d\mathbf x_{u,p}} = \sum_{(r_1,j_1)\in \mathcal I}\sum_{(r_2,j_2)\in \mathcal I}\cdots\sum_{(r_{\ell+1)},j_{\ell+1})\in \mathcal I}\sum_{w_{\ell+1}, p_{\ell+1}}  a_{r_1,j_1}(v,w_1)a_{r_2,j_2}(w_1,w_2)\cdots a_{r_{\ell+1},j_{\ell+1}}(w_{\ell},w_{\ell+1}) \\
 &\frac{d\mathbf h^{(0)}_{w_{\ell+1},p_{\ell+1}}}{d\mathbf x_{u,p}} 
    \mathbf W^{(0)}_{r_{\ell+1},j_{\ell+1}}(p_{\ell+1}, q_\ell)\mathbf W^{(1)}_{r_{\ell},j_{\ell}}(p_{\ell}, q_{\ell-1})\cdots \mathbf W^{(\ell)}_{r_{1},j_{1}}(p_{1}, q).   
\end{align*}
As $\mathbf H^{(0)}=\mathbf X$, we have
\begin{align*}
    \frac{d\mathbf h^{(0)}_{w_{\ell+1},p_{\ell+1}}}{d\mathbf x_{u,p}} =  \delta_{w_{\ell+1},u} \delta_{p_{\ell+1},p},
\end{align*}
where $\delta$ is the Kronecker delta. 
Thus we further have, under the boundedness assumption $| \mathbf W^{(\ell)}_{r,j} | \leq \beta_{r,j}$,  
\begin{align*}
    &\left |\frac{d \mathbf h^{(\ell+1)}_{v,q}}{d \mathbf x_{u,p}} \right| \leq d^{\ell+1}\sum_{(r_1,j_1)\in \mathcal I}\sum_{(r_2,j_2)\in \mathcal I}\cdots\sum_{(r_{\ell+1},j_{\ell+1})\in \mathcal I} \sum_{w_1, w_2, ..., w_{\ell}} \\
    & \left|a_{r_1,j_1}(v,w_1)\right|\left|a_{r_2,j_2}(w_1,w_2)\right|\cdots  {\left|a_{r_{\ell+1},j_{\ell+1}}(w_\ell,u)\right|}  
     \beta_{r_{1}, j_{1}}\beta_{r_{2}, j_{2}}\cdots\beta_{r_{\ell+1}, j_{\ell+1}} \\
    & =\sum_{(r_1,j_1)\in \mathcal I}\sum_{(r_2,j_2)\in \mathcal I}\cdots\sum_{(r_{\ell+1},j_{\ell+1})\in \mathcal I}\beta_{r_{1}, j_{1}}\beta_{r_{2}, j_{2}}\cdots\beta_{r_{\ell+1}, j_{\ell+1}}
     \sum_{w_1}\cdots \sum_{w_{\ell}}\left|a_{r_1,j_1}(v,w_1)\right|\cdots \left|a_{r_{\ell+1},j_{\ell+1}}(w_\ell,u)\right| d^{\ell+1}\\
    & = \sum_{(r_1,j_1)\in \mathcal I}\sum_{(r_2,j_2)\in \mathcal I}\cdots\sum_{(r_{\ell+1},j_{\ell+1})\in \mathcal I}\beta_{r_{1}, j_{1}}\beta_{r_{2}, j_{2}}\cdots\beta_{r_{\ell+1}, j_{\ell+ 1}} 
    \left(|\widehat{\mathbf A }_{r_1,j_1}|\cdot |\widehat{\mathbf A }_{r_2,j_2}| \cdots |\widehat{\mathbf A }_{r_{\ell+1},j_{\ell+1}}|\right)_{v,u} d^{\ell+1}\\
    &\leq d^{\ell+1}\beta^{\ell+1}\left(\sum_{(r_1,j_1)\in\mathcal I} \sum_{(r_1,j_1)\in\mathcal I} \cdots \sum_{(r_{\ell+1},j_{\ell+1})\in \mathcal I}\right. 
    \left.|\widehat{\mathbf A }_{r_1,j_1}|\cdot |\widehat{\mathbf A }_{r_2,j_2}| \cdots |\widehat{\mathbf A }_{r_{\ell+1},j_{\ell+1}}|\right)_{v,u}\\
    &= d^{\ell+1} \beta^{\ell+1}\left(\sum_{r,j} |\widehat{\mathbf A}_{(r,j)\in\mathcal{I}}|\right)^{\ell+1}_{u,v}.
\end{align*} 

The first inequality comes from the fact that the entries of $\widehat{\mathbf A}_{r,j}$ may not be positive. This completes the proof.
\end{proof}
Based on the above derivation, we see that the sensitivity quantified by $\left |\frac{d \mathbf h^{(\ell+1)}_{v,q}}{d \mathbf x_{u,p}} \right|$ simply depends on $\widehat{\mathbf A}_{r,j}$ when all other variables are fixed. Similar result for simplified framelet in Eq.~\eqref{simplified_framelet} is a merely a straightforward generalisation as follows. 
\begin{cor}\label{simplified_framelet_over_squash}
For simplified framelet Eq.~\eqref{simplified_framelet}, without considering activation function \guorm{(i.e., $\mathrm {softmax}$)} applied in the output layer for prediction,   we have: $\left |\frac{d \mathbf h^{(\ell+1)}_{v,q}}{d \mathbf x_{u,p}} \right| \leq d_{0}\beta \left(\sum_{(r,j) \in \mathcal I} |\widehat{\mathbf A}_{r,j}|^{\ell+1}\right)_{u,v}$.
\end{cor}
\begin{proof}
    The proof can be done by simply repeating the process of Lemma~\ref{spatial_framelet_over_squash} onto the simplified framelet Eq.~\eqref{simplified_framelet} without the crossing terms. We omit it here.
\end{proof}

\begin{rem}\label{Re:111}
It is easy to obtain similar sensitivity bounds as those in Lemma~\ref{spatial_framelet_over_squash} and Corollary~\ref{simplified_framelet_over_squash} when activation function $\sigma$ is in place in each layer. They produce a scaling factor that is bounded due to finite gradients of all activation functions currently used in GNNs. We omit this for succinctness.   
\end{rem}
Lemma~\ref{spatial_framelet_over_squash} and Corollary~\ref{simplified_framelet_over_squash} together with Remark~\ref{Re:111} state that if the activation $\sigma$ has bounded derivative and $\mathbf W_{r,J}$ has bounded elements, then the sensitivity between node features are controlled by a scaled quantity of $\sum_{(r,j) \in \mathcal I}|\mathcal W^\top_{r,j}\widehat{\mathbf A}\mathcal W_{r,j}|$. This suggests a surgery such as reweighting and/or rewiring on the graph topology that could potentially alleviate the over-squashing issue. In fact, an important indicator of the over-squashing problem that has recently been identified is the so-called \textit{graph balanced Forman curvature} \cite{ollivier2009ricci,topping2021understanding}. In \cite{topping2021understanding}, the edges with  negative balance Forman curvatures are responsible for over-squashing, and a curvature information-based graph rewiring method is proposed to tackle this problem. The core of the curvature re-wiring process is to build the tunnels (adding edges) beside the very negatively-curved \guorm{(i.e., the minimum curvature)} ones and drop those very positive-curved \guorm{(i.e., the largest curvature)} edges. Thus, one may prefer to rewire the graph according to the graph curvature information in terms of sign and quantity, if a larger $\ell$ value is required. We include an experiment to investigate the necessity of involving curvature-based rewiring for $\ell =3,4$ in Section~\ref{experiment}.

\subsection{What Knowledge Does Student Learn?}
Similar to the previous sections, we again take FMLP-O as an example and generalize to FMLP-S. Recall that in FMLP-O, we first encode all given information, $\mathcal W^\top_{0,J} \widehat{\mathbf A}\mathcal W_{0,J} $, $\mathbf X$ and $\mathcal W^\top_{r,j} \widehat{\mathbf A}\mathcal W_{r,j}$, to a unified space \guorm{with the same dimension which}  whose dimension is initially set as the same as feature dimension, i.e. $d_0$. Since there is no activation function in the first round encoding process Eq.~\eqref{first_encoding}, one can consider the process of the first encoding in the model is to learn the appropriate transformation, i.e., $\text{MLP}^{(1)}_{0,J}$, $\text{MLP}^{(1)}_{r,j}$ and $\text{MLP}^{(1)}_{\mathbf X}$,  to project everything onto the same space, a.k.a $\mathbb R^{d_0}$. 
The learned projection matrices are the incarnation of the knowledge learned by the student from digesting the teacher's graph information. Then $\boldsymbol{\alpha}$'s are generated from Eq.~\eqref{first_score_vector} to decide the relative importance between graph adjacency and feature information. To see through the operations, we decompose the matrices in the computations of $\boldsymbol{\alpha}$'s in blocks as  
\begin{align}
    &\boldsymbol{\alpha}^{(1)}_{0,J} = \mathrm{Sigmoid} ([\mathbf Q_{0,J}^{(1)}||\mathbf H^{(1)}_\mathbf X]\mathbf P^{(1)}_{0,J}+ b^{(1)}_{0,J}) = \mathrm{Sigmoid}\left([\mathbf Q_{0,J}^{(1)}||\mathbf H^{(1)}_\mathbf X]\begin{bmatrix} (\mathbf P^{(1)}_{0,J})_{\mathbf{Q}^{(1)}_{0,J}} \\ (\mathbf P^{(1)}_{0,J})_{\mathbf H^{(1)}_\mathbf X} \end{bmatrix}+b^{(1)}_{0,J}\right) \notag \\
    & =  \mathrm{Sigmoid} \left(\mathbf Q_{0,J}^{(1)}\mathbf (\mathbf P^{(1)}_{0,J})_{\mathbf{Q}^{(1)}_{0,J}} + \mathbf H^{(1)}_\mathbf X(\mathbf P^{(1)}_{0,J})_{\mathbf H^{(1)}_\mathbf X} + b^{(1)}_{0,J}\right),   
\end{align}
where the blocks $(\mathbf P^{(1)}_{0,J})_{\mathbf{Q}^{(1)}_{0,J}} \in \mathbb R^{d_0}$ and $(\mathbf P^{(1)}_{0,J})_{\mathbf H^{(1)}_\mathbf X}\in \mathbb R^{d_0}$ are the parts of matrix $\mathbf P^{(1)}_{0,J}$ that act on $\mathbf Q_{0,J}^{(1)}$ and $\mathbf H^{(1)}_\mathbf X$ respectively. Similar block structure exists in $ \boldsymbol{\alpha}^{(1)}_{r,J}$ and we have 
\[
\boldsymbol{\alpha}^{(1)}_{r,j} = \mathrm{Sigmoid} \left(\mathbf Q_{r,j}^{(1)}\mathbf (\mathbf P^{(1)}_{r,j})_{\mathbf{Q}^{(1)}_{r,j}} + \mathbf H^{(1)}_\mathbf X(\mathbf P^{(1)}_{r,j})_{\mathbf H^{(1)}_\mathbf X} + b^{(1)}_{r,j}\right).
\]Then the node representations from all frequency  domains are generated by balancing the importance of the learned knowledge. In other words, the construction of both $ \widehat{\mathbf Y}^{(1)}_{0,J}$ and $ \widehat{\mathbf Y}^{(1)}_{r,j}$ in Eq.~\eqref{first_layer_decoding} can be identified as weighted barycenter between each row of feature and adjacency information. After this knowledge distillation and ($\mathrm{MLP}$) projection, the student model has fused the knowledge provided by the teacher model into $\mathbf P_{0,J}$ and $\mathbf P_{r,J}$ and further generated score vectors that enable itself to adapt multiple types of graph datasets as we have discussed in Section~\ref{first_round_deconding}. From what we observed from the teacher model, the node representation generated from the student model after the first round decoding shall be guided by a more ``compatible'' graph information, i.e., $(\mathcal W^\top_{0,J} \widehat{\mathbf A}\mathcal W_{0,J})^{2}$ and $(\mathcal W^\top_{r,j} \widehat{\mathbf A}\mathcal W_{r,j})^{2}$. The student model repeats the previous steps with more precise guidance, which finally results in a better learning outcome. This interpretation can be applied to FMLP-S easily, we thus omit it here.

\subsection{More on \texorpdfstring{$\boldsymbol{\alpha}$}{}}
The score vectors method defined in Eq.~\eqref{first_score_vector} and Eq.~\eqref{second_score_vector}  coincides with the so-called gating network that is commonly applied in time-series and language models \cite{cho2014properties,gers2000learning,ravanelli2018light}. As we have discussed before, the quantities in $\boldsymbol{\alpha}$ give the relative importance between encoded graph connectivity and node features so that the model adapts to both homophilic and heterophilic graphs. In fact, rather than building such a gating mechanism for the student model, attempts have been made for the upstream teacher model. For example, a similar claim for heterophily adaption and escape from over-smoothing has been made in GPRGNN \cite{chien2020adaptive} which adaptively provides a weighted score before its output, that is: 
\begin{align}\label{GPRGNN}
\widehat{\mathbf H}^{(\ell+1)} = \sum_{\ell=0}^{\ell}\boldsymbol{\gamma}_\ell\mathbf H^{(\ell)},
\end{align}
where $\boldsymbol{\gamma}_\ell \in \mathbb R^{N\times N}$ is learnable generalized Page rank coefficients. One can easily check that the weighting scheme in GPRGNN~\ref{GPRGNN} has  similar functionality of adding a gating mechanism based on the importance of the GNN output. However, compared to GPRGNN, our student model treats the graph knowledge individually and generates the weighting score $\boldsymbol{\alpha}$ simply based on the numerical features of the encoded knowledge, and hence has more intuitive  interpretation of the learned graph knowledge. 

\subsection{Encoding Dimension} 
The difference between the number of nodes and the initial feature dimension varies in real-world benchmarks. For example, one can observe that for most of the graph datasets, the number of nodes is larger than the feature dimension. For example,  \textbf{Cora} contains 2708 nodes with feature dimension 1433.  Whereas there are indeed some graph benchmarks with its feature dimension higher than the number of nodes. For example, \textbf{Citeseer} contains 3327 nodes with each node being 3703 dimensional. Nevertheless, we found that the first encoding dimension, initially set as $d_0$ in Eq.~\eqref{first_encoding}, significantly affects the learning speed of the student model. Regardless whether $d_0$ is larger or less than the number of nodes, it is still a relatively large quantity in most of the cases. In fact, by assuming the graph 
is generated from an underlying low-dimensional manifold \cite{rubin2020manifold,tang2013universally,lei2021network}, the encoding dimension can be much lower than $d_0$ 
without losing too much information. This suggests setting a lower first encoding dimension in practice to further speed up the student model, although the discussion on the optimal such dimension is out of the scope of this paper.

\section{Experiment}\label{experiment}
In this section, we present comprehensive experimental results to demonstrate the performances of the proposed models, FMLP-O and FMLP-S, on both synthetic and real-world homophilic and heterophilic graph datasets. In the synthetic experiment, we verify that the score vectors  $\boldsymbol{\alpha}$ in student models are capable of balancing the relative importance between nodes features and adjacency information and thus enhancing the model's adaption for both homophily and heterophily datasets. Then we show the learning outcomes of our student models in both real-world homophily and heterophily datasets.  Lastly we test the performance of our models when the layer number becomes high, i.e., $\ell =3$ and $4$, and show that by applying the rewiring method established on  Stochastic Discrete Ricci Flow (SDRF) \cite{topping2021understanding}, both FMLP-O and FMLP-S are able to achieve the performance identical to or higher than that of their teacher models. All experiments \footnote{Source code of this paper can be found in \url{https://github.com/dshi3553usyd/Frameless_Graph_Distillation.git}} were conducted using PyTorch on NVIDIA\textsuperscript{\textregistered} Tesla V100 GPU with 5,120 CUDA cores and 16GB HBM2 mounted on an HPC cluster.

\subsection{Synthetic Experiments}\label{Synthetic}
\paragraph{Set up}To show our proposed models' adaption to homophilic and heterophilic graph datasets, followed by \cite{di2022graph,zhu2020beyond}, we generate the synthetic Cora dataset with the controllable homophily index $ \mathbf H(\mathcal G) \in [0,1]$. $ \mathbf H(\mathcal G)=0$ means that the graph is totally heterophily, and $ \mathbf H(\mathcal G)=1$   means perfect homophily. We report the performance of both teacher (named as SpatUFG and SUFG for spatial framelet and simplified framelet) and student models in Table~\ref{synthetic_cora_result} together with the mean value of the score vectors that are applied in student models to balance the relative importance between feature and adjacency information of the graph. To save space, we included the detailed setup information of both teacher and student models in section \ref{homo_experiment} and \ref{heterophily_experiment}.

\begin{table}[H]
\centering
\caption{Performances of both teacher and student models in synthetic Cora datasets, where $\boldsymbol{\alpha}$\_\textbf{lowpass} and $\boldsymbol{\alpha}$\_\textbf{highpass} are the mean value of the score vectors involved in student models. For FML-O, the $\boldsymbol{\alpha}$s are the mean values of the score vectors in the first decoding process (i.e., $\boldsymbol{\alpha}^{(1)}_{0,J}$ and $\boldsymbol{\alpha}^{(1)}_{r,j}$). It is worth noting that for all experiments, we set $r=J =1$ for both teacher and student models, indicating there will be only one low and high-pass domain.}
\setlength{\tabcolsep}{3.5pt}
\renewcommand{\arraystretch}{1.4}
\label{synthetic_cora_result}
\begin{tabular}{llllll}
\hline
\textbf{Datasets}   &\textbf{Cora 0.0}      & \textbf{Cora 0.1} & \textbf{Cora 0.3} & \textbf{Cora 0.5} & \textbf{Cora 0.8} \\ \hline
\textbf{SUFG}    & 0.31\(\pm\)0.33       &   0.34\(\pm\)0.26                &   0.47\(\pm\)0.15                &   0.64\(\pm\)0.11                  &  0.81\(\pm\)0.44                 \\
\textbf{FMLP-S}   & 0.45\(\pm\)0.13       &   0.46\(\pm\)0.35                &   0.57\(\pm\)0.07                & 0.74\(\pm\)0.05               &  0.82\(\pm\)0.33                 \\
\textbf{$\boldsymbol{\alpha}$\_lowpass} & \multicolumn{1}{c}{0.589}    &
 \multicolumn{1}{c}{0.578}                & \multicolumn{1}{c}{0.523}                    &  \multicolumn{1}{c}{0.501}                   &  \multicolumn{1}{c}{0.483}            \\
\textbf{$\boldsymbol{\alpha}$\_highpass} &
\multicolumn{1}{c}{0.591}   &   
\multicolumn{1}{c}{0.572}                & \multicolumn{1}{c}{0.520}                    &  \multicolumn{1}{c}{0.503}                   &  \multicolumn{1}{c}{0.464}                  \\ \hline
\textbf{SpatUFG}  & 0.30\(\pm\)0.29

&     0.39\(\pm\)0.43                &   0.51\(\pm\)0.25                &   0.72\(\pm\)0.41                  &  0.83\(\pm\)0.44                          \\
\textbf{FMLP-O}   &  0.46\(\pm\)0.37      &   0.48\(\pm\)0.27                &   0.56\(\pm\)0.17                &  0.74\(\pm\)0.37                    &   0.84\(\pm\)0.25                 \\
\textbf{$\boldsymbol{\alpha}$\_lowpass}  &   
\multicolumn{1}{c}{0.567}  &
\multicolumn{1}{c}{0.549}                & \multicolumn{1}{c}{0.510}                    &  \multicolumn{1}{c}{0.482}                   &  \multicolumn{1}{c}{0.441}        \\
\textbf{$\boldsymbol{\alpha}$\_highpass}&
\multicolumn{1}{c}{0.574}            &  
\multicolumn{1}{c}{0.540}                & \multicolumn{1}{c}{0.517}                    &  \multicolumn{1}{c}{0.475}                   &  \multicolumn{1}{c}{0.423} \\
\hline   
\end{tabular}
\end{table}

\paragraph{Results}
Both teacher and student models present relatively low learning accuracies when the quantity of $ \mathbf H(\mathcal G)$ is low. However, one can directly check that both student models outperform their teacher models, especially for datasets with low $ \mathbf H(\mathcal G)$ index values. Followed by the increase of $ \mathbf H(\mathcal G)$ index from $0.0$ to $0.8$, the mean value of both low and high-pass $\boldsymbol{\alpha}$ decreases, suggesting that models rely more on the adjacency information rather than feature information when the graph becomes more homophilic. This observation directly supports our claim that student models are equipped with a higher heterophily adaption (from $\boldsymbol{\alpha}$) in Section~\ref{section_decoding}.

\subsection{Performance on Homophilic Graphs}\label{homo_experiment} 
\begin{table}[t]
\centering
\caption{Summary statistics of eight included homophily graph datasets, the column \textbf{H}($\mathcal G$) is the homophily index.}
\setlength{\tabcolsep}{2pt}
\renewcommand{\arraystretch}{1.4}
\label{tab:homo_data_summary}
\begin{tabular}{lccccc}
\hline
\textbf{Datasets}   & \multicolumn{1}{l}{\textbf{\#Classes}} & \multicolumn{1}{l}{\textbf{\#Features}} & \multicolumn{1}{l}{\textbf{\#nodes}} & \multicolumn{1}{l}{\textbf{\#Edges}}  &  \multicolumn{1}{l}{\textbf{H($\mathcal G$)}}    \\ \hline
\textbf{Cora}       & 7                                      & 1433                                    & 2708                                 & 5429             &  0.825                 \\
\textbf{Citeseer}   & 6                                      & 3703                                    & 3327                                 & 4372              &   0.717                \\
\textbf{Pubmed}     & 3                                      & 500                                     & 19717                                & 44338              &       0.792           \\
\textbf{CS}         & 15                                     & 6805                                    & 18333                                & 100227              &     0.832            \\
\textbf{Physics}    & 5                                      & 8415                                    & 34493                                & 495924               &         0.915       \\
\textbf{Computer}   & 10                                     & 767                                     & 13381                                & 259159                &          0.802     \\
\textbf{Photo}      & 8                                      & 745                                     & 7487                                 & 126530                 &      0.849        \\
\textbf{ogbn-arxiv} & 47                                     & 100                                     & 169343                               & 1166243             &              0.681   \\ \hline
\end{tabular}
\end{table}

\paragraph{Set up} We tested the performance of both FMLP-O and FMLP-S on eight benchmark datasets whose statistics are summarized in Table~\ref{tab:homo_data_summary}. We included seven citation networks, \textbf{Cora, Citeseer, Pubmed, Coauthor-physics, Coauthor-CS}, and \textbf{Amazon-photo}, and one large-scale graph dataset \textbf{ogbn-arxiv}. We applied $\mathrm{Plantoid}$ split for citation networks and public split for \textbf{ogbn} dataset. Additionally, we split the rest of benchmarks to $20\%, 20\% ,60\%$  of nodes per class for train, validation, and test proportions based on \cite{pei2020geom} for a fair comparison. Both teacher and student models hyperparameters were set to default, except for learning rate, weight decay, hidden units, dropout ratio, and the encoding dimension in training. The grid search  was applied to fine-tune  these hyperparameters. The search space for learning rate was in $\{0.1, 0.05, 0.01, 0.005\}$, number of hidden units in $\{16, 32, 64\}$,weight decay in $\{0.05, 0.01,0.1,0.5\}$ and the encoding dimension was set in $\{128, 64, 32, 16,8\}$. We applied $\mathrm{Haar}$-type frame for all framelet model implementations, see \cite{yang2022quasi} for details. In addition, we set the maximum number of epochs of 200 for citation networks and 500 for \textbf{ogbn-arxiv}. The average test accuracy and its standard deviation came from 10 runs.

\paragraph{Baselines}
In terms of the baseline models, we first set framelet Eq.~\eqref{spatial_framelt_teacher}  and simplified framelet Eq.~\eqref{simplified_framelet}
as the teacher models in FMLP-O and FMLP-S. Specifically, in this experiment, the framelets models have two framelet convolution layers ($\ell =2$). We also set $J=1$ for the ($\mathrm{Haar}$) filtering functions, indicating only one low-pass and one high-pass domain.
We also applied the same setting to the simplified framelet.
Furthermore, we compared the performances of our models with four additional based line models: MLP, student $\mathrm{MLP}$, denoted as $\mathrm{MLP}^S$, and  GLNN \cite{zhang2021graph} with the later (GLNN) first introduce the $\mathrm{MLP}$ based distillation scheme to GNNs. The learning outcomes are reported from the published results with the same experimental setup if available. If the results were not previously reported, we performed a hyper-parameter search based on the official implementations.

\begin{table*}[t]
\centering
\caption{Performance of FMLP-O and FMLP-S on eight homophilic graphs, the result will be highlighted in \textbf{bold} once the student models outperform their teacher models}
\setlength{\tabcolsep}{3.5pt}
\renewcommand{\arraystretch}{1.4}
\label{tab:homophily_performance}
\begin{tabular}{cllllllll}
\hline
\textbf{Datasets}   & \multicolumn{1}{c}{\textbf{SpatUFG}} & \multicolumn{1}{c}{\textbf{FMLP-O}} & \multicolumn{1}{c}{\textbf{SUFG}} & \multicolumn{1}{c}{\textbf{FMLP-S}} & \multicolumn{1}{c}{\textbf{MLP}} & \multicolumn{1}{c}{$\textbf{MLP}^S$} & \multicolumn{1}{c}{\textbf{GLNN}}  \\ \hline
\textbf{Cora}       &      83.5\(\pm\)0.45                                &     \textbf{83.7\(\pm\)0.53}                                &     82.9\(\pm\)0.26                                &    \textbf{83.4\(\pm\)0.11}                                 &  59.2\(\pm\)1.34                           &     83.0\(\pm\)0.48                                                     &   83.6\(\pm\)0.22                                &                                     \\
\textbf{Citeseer}   &       74.0\(\pm\)0.36                               &          \textbf{74.5\(\pm\)0.71}                          &    73.8\(\pm\)0.73                               &     \textbf{74.3\(\pm\)0.15}                                & 59.6\(\pm\)2.97                                  &    72.8\(\pm\)1.22                                                 &        74.0\(\pm\)0.02                           &                                     \\
\textbf{Pubmed}     &    79.4\(\pm\)0.64                                 &   \textbf{79.9\(\pm\)0.97}                                  &   79.0\(\pm\)0.27                                &    79.0\(\pm\)0.18                                 &       67.6\(\pm\)2.32                           &  75.2\(\pm\)0.51                                                   &    79.6\(\pm\)0.14                               &                                     \\
\textbf{Photo}      &   89.2\(\pm\)0.41                                   &   \textbf{90.2\(\pm\)0.74}                                  &    89.0\(\pm\)1.14                               &    \textbf{89.9\(\pm\)0.32}                                 &    69.6\(\pm\)3.92                              & 77.9\(\pm\)1.38                                                    &    84.5\(\pm\)1.87                               &                                     \\
\textbf{Computer}   &   83.6\(\pm\)0.16                                   &      \textbf{84.1\(\pm\)0.75}                               & 81.9\(\pm\)0.85                                  &    81.6\(\pm\)0.82                                 &  44.9\(\pm\)0.88                                &     70.2\(\pm\)0.67                                                &                82.6\(\pm\)1.34                   &                                     \\
\textbf{CS}         & 92.3\(\pm\)0.44                                     &   \textbf{92.8\(\pm\)0.37}                                  &   91.9\(\pm\)1.25                                &  \textbf{92.3\(\pm\)0.11}                                   &    88.3\(\pm\)0.77                               & 80.2\(\pm\)3.21                                                    &    90.3\(\pm\)0.81                               &                                     \\
\textbf{Physics}    &   93.4\(\pm\)1.23                                   &      \textbf{94.1\(\pm\)0.42}                               &    92.8\(\pm\)0.62                               &      \textbf{93.2\(\pm\)0.37}                               &    88.9\(\pm\)1.14                              &76.0\(\pm\)2.25                                                     &            91.6\(\pm\)0.49                       &                                     \\
\textbf{ogbn-arxiv} &     73.2\(\pm\)3.42                                 &     \textbf{73.3\(\pm\)1.25}                                &   71.0\(\pm\)1.38                                & 70.9\(\pm\)0.66                                    &    56.0\(\pm\)0.54                              &                   68.9\(\pm\)2.31                                  &     70.8\(\pm\)0.27                              &                                     \\ \hline
\end{tabular}
\end{table*}

\paragraph{Results}
The performance of both FMLP-O and FMLP-S are listed in Table~\ref{tab:homophily_performance}. Our student models outperform their teacher models in most of included benchmarks. Furthermore, the outcomes of the simplified framelet model (SUFG) are less accurate than the spatial framelet (SpatUFG), and this observation is aligned with the empirical observations provided in \cite{wu2019simplifying}. In addition, the accuracy difference between teacher models is inherited by the corresponding student models, as one can observe that FMLP-O surpasses FMLP-S in all datasets.
These observations lead us to conjecture whether a better teacher can lead to a better student. We found several empirical explorations \cite{qiu2022better,zhu2021student} from the existing literature have been attempted in this field, although the discussion of this is out of the scope of this paper. Moreover, compared to $\mathrm{MLP}^S$, which is utilized to imitate the student models' outcomes, the GLNN model \cite{zhang2021graph} shows its advantage from training with mixed graph knowledge (i.e., $\mathbf X$ and $\widehat{\mathbf A}$) and results in higher performance than $\mathrm{MLP}^S$. However, one can observe that the prediction accuracy of GLNN is generally lower than our proposed student models due to the differences in graph knowledge.

\subsection{Performance on Heterophily Graphs}\label{heterophily_experiment}

\paragraph{Set up}
In this experiment, we tested the performance of our student models on heterophilic graphs. Specifically, the input datasets are \textbf{Chameleon, Squirrel, Actor, Wisconsin, Texas} and \textbf{Cornell}, the summary statistics, including the homophily index \cite{pei2020geom}, and the split ratio between train, test, and validation of these datasets are included in Table~\ref{tab:hetero_data_summary}. Similar to the settings for homophily graphs, we let spatial framelet and simplified framelet as teacher models and conducted a grid search for hyper-parameters. In addition, based on what we have discussed in Section~\ref{FLMPS_and_discussions}, we trained our teacher models with energy perturbation strategy \cite{chendirichlet} with two additional hyperparameters $\boldsymbol{\epsilon}$ for spatial framelet and $\boldsymbol{\epsilon}^S$ for simplified framelet respectively. We searched the optimal quantity (which might differ for different datasets) of two $\boldsymbol{\epsilon}$'s in the parameter searching space. It is worth noting that the graph knowledge supplied by the teacher models (i.e., $\mathcal W_{r,j}^\top \widehat{\mathbf A} \mathcal W_{r,j}$) is changed accordingly to the energy perturbation.

\begin{table}[t]
\caption{Summary statistics of the datasets, \textbf{H}($\mathcal G$) represent the homophily index of the included datasets. All datasets are followed with a public split ratio which includes 60\% for training, 20\% for testing and validation.}
\setlength{\tabcolsep}{2pt}
\renewcommand{\arraystretch}{1.4}
\centering
    \begin{tabular}{c c c c c c c c c}
        \hline
         \textbf{Datasets} & \textbf{\#Class} & \textbf{\#Feature} & \textbf{\#Node}& \textbf{\#Edge} & \textbf{H($\mathcal G$)}\\
         \hline
        \textbf{Chameleon} &5 &2325 &2277 &31371  &0.247\\
        \textbf{Squirrel} &5 &2089 &5201 &198353  &0.216\\
        \textbf{Actor} &5 &932 &7600 &26659  &0.221\\
        \textbf{Wisconsin} &5 &251 &499 &1703  &0.150\\
        \textbf{Texas} &5 &1703 &183 &279  &0.097\\
        \textbf{Cornell} &5 &1703 &183 &277  &0.386\\
        \hline
    \end{tabular}
    
    \label{tab:hetero_data_summary}
\end{table}

\begin{table*}[t]
\centering
\caption{Performance of FMLP-O and FMLP-S on heterophilic graphs. Learning accuracies are highlighted in \textbf{bold} once the student model outperforms its teacher model.}
\setlength{\tabcolsep}{3.5pt}
\renewcommand{\arraystretch}{1.2}
\label{tab:heterophily_performance}
\begin{tabular}{cllllllll}
\hline
\textbf{Datasets}  &\multicolumn{1}{c}{\textbf{MLP}}   & \multicolumn{1}{c}{\textbf{SpatUFG}} 
& \multicolumn{1}{c}{\textbf{FMLP-O}} & \multicolumn{1}{c}{\textbf{SUFG}} 
 & \multicolumn{1}{c}{\textbf{FMLP-S}} 
 &  \multicolumn{1}{c}{$\textbf{MLP}^S$} & \multicolumn{1}{c}{\textbf{GLNN}}  \\ \hline
\textbf{Cornell}   &  90.3\(\pm\)0.70  &     82.2\(\pm\)0.28         &   \textbf{84.1\(\pm\)0.73}                    &         81.2\(\pm\)0.83                          &      \textbf{81.3\(\pm\)0.37}                      & 70.6\(\pm\)0.48                                                       &     74.8\(\pm\)0.89                             &                                     \\
\textbf{Texas}   &     91.3\(\pm\)0.71          &   85.3\(\pm\)0.54   & 85.1\(\pm\)0.16                &          84.5\(\pm\)0.40                          &     \textbf{84.9\(\pm\)0.24}                                                                                        &     73.4\(\pm\)0.78                   & 79.1\(\pm\)0.23                                  &                                     \\
\textbf{Wisconsin}     & 91.8\(\pm\)3.33          &   90.4\(\pm\)0.25                      &  \textbf{90.6\(\pm\)0.12}                                  &    88.1\(\pm\)0.23                              &   \textbf{88.2\(\pm\)0.46}                                                      &      74.3\(\pm\)1.26                                                 &     87.3\(\pm\)0.37                            &                                     \\
\textbf{Actor}      &    38.5\(\pm\)0.25        &      42.5\(\pm\)0.19       &  42.5\(\pm\)0.11            &    40.1\(\pm\)0.20                &  \textbf{41.3\(\pm\)0.24}                                &   36.7\(\pm\)1.12                               &  37.8\(\pm\)0.47                                                                                                                       &                                     \\
\textbf{Squirrel}   &      31.3\(\pm\)0.27        &     51.6\(\pm\)0.33       &   \textbf{51.9\(\pm\)0.22}         &       46.4\(\pm\)0.19      &    46.4\(\pm\)0.35                         &   36.3\(\pm\)1.38                                &    41.8\(\pm\)1.52                                                                                                                     &                                     \\
\textbf{Chameleon}         &  46.7\(\pm\)0.46        &    58.0\(\pm\)0.28        &    \textbf{58.7\(\pm\)0.32}             &             55.6\(\pm\)0.27                       &    \textbf{56.2\(\pm\)0.16}                                                                &   48.9\(\pm\)0.82                                                                                   &   50.7\(\pm\)0.41                                &                                           \\ \hline
\end{tabular}
\end{table*}

\paragraph{Results}
The learning accuracies of both teacher and student models are included in Table~\ref{tab:heterophily_performance}. One can observe that both teacher models achieved remarkable accuracy compared to the heterophily graph baseline ($\mathrm{MLP}$). 
In fact, it has been illustrated that the spatial framelet can naturally adapt to heterophily graph asymptotically (i.e., $\ell$ is large) when there is a high-frequency dominant dynamic induced from spatial framelet \cite{han2022generalized}. Similar to the results for homophily experiments, both student models outperform teacher models to some extents, and this indicates the natural advantage of $\mathrm{MLP}$ based models on heterophily graphs. 
Finally, compared to the prediction accuracy in homophily experiments, one can see that the GLNN is no longer powerful for heterophily graph datasets, suggesting the effectiveness of supplying and distilling the graph knowledge individually and independently when the input graph is highly heterophilic.

\subsection{Performance on Curvature Rewired Graphs}
\paragraph{Set up} In this section, we show how the additional graph surgery (i.e., curvature-based rewiring) can help both teacher and student models to generate better prediction outcomes when the number of layers is high. Specifically, we focus on $\ell =3$ and  $4$ on both our teacher and student models with and without graph rewiring. All models' parameters were set as the same as we included in the homophily and heterophily experiments. In addition, we also included the number of added/deleted edges from graph rewiring scheme to explicitly show how the graph connectivity was changed as SDRF cut the edges with very positive (Balanced Forman) curvature and built tunnels beside the edges with very negative curvature \cite{topping2021understanding}. We select \textbf{Cora, Citeseer}, \textbf{Wisconsin}, and \textbf{Actor} datasets for illustration, and the results are contained in Table~\ref{tab:sdrf}. We note that the graph rewiring approach was deployed before the initiation of the model rather than applied via each layer during the model training, and thus served as one data prep-processing approach. Lastly, we set the energy perturbation quantities: $\boldsymbol{\epsilon}$ and $\boldsymbol{\epsilon}^S$ as zero to exclude the heterophily adaption advantage of teacher models so that the over-squashing effect can be isolated.

\begin{table*}[t]
\centering
\caption{Performance of SUFG and its student model FMLP-S before and after graph knowledge rewiring. $\Delta$Edges standards for the changes of edges number after the SDRF. The upper part of the table shows the learning outcome when $\ell =3$ and the lower part presents the outcome when $\ell =4$. Models with stars are with the input data processed after SDRF to alleviate the over-squashing issue. Results will be highlighted in \textbf{bold} once the models enhanced SDRF outperforms its original counterparts.}
\setlength{\tabcolsep}{4.5pt} 
\renewcommand{\arraystretch}{1} 
\label{tab:sdrf}
\scalebox{0.9}{
\begin{tabular}{ccccc}

\hline
\textbf{Methods}      & \textbf{Cora} & \textbf{Citeseer} & \textbf{Wisconsin} & \textbf{Actor} \\ \hline
\textbf{$\Delta$Edges}        &   508            &    -90               &   15                 &    32              \\
\textbf{SpatUFG}         &     80.5\(\pm\)0.81          &  68.0\(\pm\)0.22                  &    76.3\(\pm\)0.40               &    30.9\(\pm\)0.27              \\
\textbf{SpatUFG*}   &  \textbf{80.9\(\pm\)0.25}              &   \textbf{68.6\(\pm\)0.31}                &  \textbf{78.2\(\pm\)0.45}                 & \textbf{32.0\(\pm\)0.28}                 \\
\textbf{FMLP-O}       &   80.9\(\pm\)0.18            & 70.8\(\pm\)0.35                   &        77.1\(\pm\)0.36             &    31.1\(\pm\)0.16             \\
\textbf{FMLP-O*} &  \textbf{81.1\(\pm\)0.41}             & \textbf{71.1\(\pm\)0.32}                   &  \textbf{78.7\(\pm\)0.21}                    &  \textbf{32.3\(\pm\)0.41}                \\
\textbf{SUFG}         &    81.3\(\pm\)0.23        &     70.4\(\pm\)0.30              & 74.3\(\pm\)0.56                 &  32.9\(\pm\)0.50                \\
\textbf{SUFG*}   &    81.3\(\pm\)0.27&  \textbf{70.9\(\pm\)0.25}                 &    \textbf{74.7\(\pm\)0.13}                &   \textbf{33.6\(\pm\)0.18}               \\
\textbf{FMLP-S}       &   81.4\(\pm\)0.14            &    70.5\(\pm\)0.42               &   74.2\(\pm\)0.28                 &    33.9\(\pm\)0.59              \\
\textbf{FMLP-S*} &      \textbf{ 81.8\(\pm\)0.18}        &        70.5\(\pm\)0.28           &  \textbf{75.0\(\pm\)0.32}                  &    \textbf{35.6\(\pm\)0.61}              \\ \hline
\textbf{SpatUFG}         &  73.9\(\pm\)0.59           &   61.9\(\pm\)0.49            &  81.3\(\pm\)0.37                &    27.9\(\pm\)0.27                \\
\textbf{SpatUFG*}   & \textbf{75.1\(\pm\)0.11
}                 &   \textbf{62.6\(\pm\)0.25}                 &   \textbf{82.0\(\pm\)0.44}                 &    \textbf{29.4\(\pm\)0.71 }                \\
\textbf{FMLP-O}       &    74.0\(\pm\)0.12           &   62.0\(\pm\)0.13                &   81.4\(\pm\)0.24                 &     28.2\(\pm\)0.19              \\
\textbf{FMLP-O*} &   \textbf{75.8\(\pm\)0.34}             &  \textbf{63.0\(\pm\)0.35}                 & \textbf{81.9\(\pm\)0.19}                   &   \textbf{29.9\(\pm\)0.21}  

\\
\textbf{SUFG}         &   80.8\(\pm\)0.49            & 69.9\(\pm\)0.72                  &    74.4\(\pm\)0.29                & 29.7\(\pm\)0.14               \\
\textbf{SUFG*}   &     80.7\(\pm\)0.35          &   \textbf{70.4\(\pm\)0.35}                &     \textbf{75.1\(\pm\)0.17}               &  \textbf{30.0\(\pm\)0.26}             \\
\textbf{FMLP-S}       &  80.9\(\pm\)0.15             &  70.1\(\pm\)0.24                 &    74.8\(\pm\)0.13                &    29.9\(\pm\)0.24              \\
\textbf{FMLP-S*} &    \textbf{81.3\(\pm\)0.14}           &      \textbf{70.6\(\pm\)0.17}             &     \textbf{75.6\(\pm\)0.29}               &     \textbf{33.4\(\pm\)0.32}             \\ \hline
\end{tabular}}
\end{table*}

\paragraph{Results}
Similar to the results from the previous experiments, our student models outperform their teacher models. In terms of the teacher models outcomes, compared to the result with $\ell =3$, spatial framelet shows even higher prediction accuracy when $\ell =4$, and this again aligns with the recent conclusion on the asymptotic behavior of framelet \cite{han2022generalized}. Although in most of the results, models with data pre-processed by SDRF are usually with a higher prediction accuracy than their original counterparts, this phenomenon is more evident via heterophily graphs (\textbf{Wisconsin} and \textbf{Actor}), suggesting the effectiveness of incorporating proper graph topological modification to both teacher and student models.







\section{Conclusion}\label{Conclusion}
This work showed how the (offline) graph knowledge distillation scheme can be applied to multi-scale graph neural networks (graph framelets). Together with the teacher model induced from the original graph framelet, we proposed a simplified graph framelet as one additional teacher model and discussed its property in detail. Meanwhile, we also provided a comprehensive analysis of both teacher and student models, including how the student learns and digests the teacher model's graph knowledge and why the student model can adapt to both homophilic and heterophilic graph datasets. We also discussed the potential computational issues, and how to alleviate them. Our experiments showed that the simplified framelet achieved nearly identical learning outcomes to the original version of the teacher model, and two $\mathrm{MLP}$ based student models outperform their teacher models on most graph datasets. Further research may include optimizing the graph knowledge supplied by the teacher model as well as the student model's architecture to improve the performance even further.

\bibliographystyle{plain}
\bibliography{ref}

\begin{thebibliography}{10}

\bibitem{arora2016understanding}
Raman Arora, Amitabh Basu, Poorya Mianjy, and Anirbit Mukherjee.
\newblock Understanding deep neural networks with rectified linear units.
\newblock {\em arXiv preprint arXiv:1611.01491}, 2016.

\bibitem{chen2020online}
Defang Chen, Jian-Ping Mei, Can Wang, Yan Feng, and Chun Chen.
\newblock Online knowledge distillation with diverse peers.
\newblock In {\em Proceedings of the AAAI Conference on Artificial
  Intelligence}, volume~34, pages 3430--3437, 2020.

\bibitem{chendirichlet}
Jialin Chen, Yuelin Wang, Cristian Bodnar, Pietro Li{\`o}, and Yu~Guang Wang.
\newblock Dirichlet energy enhancement of graph neural networks by framelet
  augmentation.

\bibitem{chen2022sa}
Jie Chen, Shouzhen Chen, Mingyuan Bai, Junbin Gao, Junping Zhang, and Jian Pu.
\newblock Sa-mlp: Distilling graph knowledge from gnns into structure-aware
  mlp.
\newblock {\em arXiv preprint arXiv:2210.09609}, 2022.

\bibitem{chen2021graph}
Jie Chen, Shouzhen Chen, Mingyuan Bai, Jian Pu, Junping Zhang, and Junbin Gao.
\newblock Graph decoupling attention markov networks for semisupervised graph
  node classification.
\newblock {\em IEEE Transactions on Neural Networks and Learning Systems},
  pages 1--15, 2022.

\bibitem{chen2020self}
Yuzhao Chen, Yatao Bian, Xi~Xiao, Yu~Rong, Tingyang Xu, and Junzhou Huang.
\newblock On self-distilling graph neural network.
\newblock {\em arXiv preprint arXiv:2011.02255}, 2020.

\bibitem{chien2020adaptive}
Eli Chien, Jianhao Peng, Pan Li, and Olgica Milenkovic.
\newblock Adaptive universal generalized pagerank graph neural network.
\newblock {\em arXiv preprint arXiv:2006.07988}, 2020.

\bibitem{cho2014properties}
Kyunghyun Cho, Bart Van~Merri{\"e}nboer, Dzmitry Bahdanau, and Yoshua Bengio.
\newblock On the properties of neural machine translation: Encoder-decoder
  approaches.
\newblock {\em arXiv preprint arXiv:1409.1259}, 2014.

\bibitem{chung1997spectral}
Fan~RK Chung.
\newblock {\em Spectral graph theory}, volume~92.
\newblock American Mathematical Soc., 1997.

\bibitem{defferrard2016convolutional}
Micha{\"e}l Defferrard, Xavier Bresson, and Pierre Vandergheynst.
\newblock Convolutional neural networks on graphs with fast localized spectral
  filtering.
\newblock {\em Advances in neural information processing systems}, 29, 2016.

\bibitem{di2022graph}
Francesco Di~Giovanni, James Rowbottom, Benjamin~P Chamberlain, Thomas
  Markovich, and Michael~M Bronstein.
\newblock Graph neural networks as gradient flows.
\newblock {\em arXiv preprint arXiv:2206.10991}, 2022.

\bibitem{dong2017sparse}
Bin Dong.
\newblock Sparse representation on graphs by tight wavelet frames and
  applications.
\newblock {\em Applied and Computational Harmonic Analysis}, 42(3):452--479,
  2017.

\bibitem{gers2000learning}
Felix~A Gers, J{\"u}rgen Schmidhuber, and Fred Cummins.
\newblock Learning to forget: Continual prediction with lstm.
\newblock {\em Neural computation}, 12(10):2451--2471, 2000.

\bibitem{gilmer2017neural}
Justin Gilmer, Samuel~S Schoenholz, Patrick~F Riley, Oriol Vinyals, and
  George~E Dahl.
\newblock Neural message passing for quantum chemistry.
\newblock In {\em International conference on machine learning}, pages
  1263--1272. PMLR, 2017.

\bibitem{gou2021knowledge}
Jianping Gou, Baosheng Yu, Stephen~J Maybank, and Dacheng Tao.
\newblock Knowledge distillation: A survey.
\newblock {\em International Journal of Computer Vision}, 129(6):1789--1819,
  2021.

\bibitem{hamilton2017inductive}
Will Hamilton, Zhitao Ying, and Jure Leskovec.
\newblock Inductive representation learning on large graphs.
\newblock {\em Advances in neural information processing systems}, 30, 2017.

\bibitem{hammond2011wavelets}
David~K Hammond, Pierre Vandergheynst, and R{\'e}mi Gribonval.
\newblock Wavelets on graphs via spectral graph theory.
\newblock {\em Applied and Computational Harmonic Analysis}, 30(2):129--150,
  2011.

\bibitem{han2022generalized}
Andi Han, Dai Shi, Zhiqi Shao, and Junbin Gao.
\newblock Generalized energy and gradient flow via graph framelets.
\newblock {\em arXiv preprint arXiv:2210.04124}, 2022.

\bibitem{hinton2015distilling}
Geoffrey Hinton, Oriol Vinyals, Jeff Dean, et~al.
\newblock Distilling the knowledge in a neural network.
\newblock {\em arXiv preprint arXiv:1503.02531}, 2(7), 2015.

\bibitem{hou2019learning}
Yuenan Hou, Zheng Ma, Chunxiao Liu, and Chen~Change Loy.
\newblock Learning lightweight lane detection cnns by self attention
  distillation.
\newblock In {\em Proceedings of the IEEE/CVF international conference on
  computer vision}, pages 1013--1021, 2019.

\bibitem{ji2021survey}
Shaoxiong Ji, Shirui Pan, Erik Cambria, Pekka Marttinen, and S~Yu Philip.
\newblock A survey on knowledge graphs: Representation, acquisition, and
  applications.
\newblock {\em IEEE transactions on neural networks and learning systems},
  33(2):494--514, 2021.

\bibitem{joulin2016fasttext}
Armand Joulin, Edouard Grave, Piotr Bojanowski, Matthijs Douze, H{\'e}rve
  J{\'e}gou, and Tomas Mikolov.
\newblock Fasttext. zip: Compressing text classification models.
\newblock {\em arXiv preprint arXiv:1612.03651}, 2016.

\bibitem{ke2020rethinking}
Guolin Ke, Di~He, and Tie-Yan Liu.
\newblock Rethinking positional encoding in language pre-training.
\newblock {\em arXiv preprint arXiv:2006.15595}, 2020.

\bibitem{khesin2014arnold}
Boris~A Khesin and Serge~L Tabachnikov.
\newblock {\em ARNOLD: Swimming Against the Tide: Swimming Against the Tide},
  volume~86.
\newblock American Mathematical Society, 2014.

\bibitem{kipf2016semi}
Thomas~N Kipf and Max Welling.
\newblock Semi-supervised classification with graph convolutional networks.
\newblock {\em arXiv preprint arXiv:1609.02907}, 2016.

\bibitem{lee2019rethinking}
H~Lee, SJ~Hwang, and J~Shin.
\newblock Rethinking data augmentation: Self-supervision and self-distillation.
  arxiv 2019.
\newblock {\em arXiv preprint arXiv:1910.05872}, 2019.

\bibitem{lei2021network}
Jing Lei.
\newblock Network representation using graph root distributions.
\newblock {\em arXiv preprint arXiv: 1802.09684}, 2021.

\bibitem{li2020fast}
Ming Li, Zheng Ma, Yu~Guang Wang, and Xiaosheng Zhuang.
\newblock Fast {H}aar transforms for graph neural networks.
\newblock {\em Neural Networks}, 128:188--198, 2020.

\bibitem{li2018deeper}
Qimai Li, Zhichao Han, and Xiao-Ming Wu.
\newblock Deeper insights into graph convolutional networks for semi-supervised
  learning.
\newblock In {\em AAAI Conference on Artificial Intelligence}, 2018.

\bibitem{ollivier2009ricci}
Yann Ollivier.
\newblock Ricci curvature of markov chains on metric spaces.
\newblock {\em Journal of Functional Analysis}, 256(3):810--864, 2009.

\bibitem{pei2020geom}
Hongbin Pei, Bingzhe Wei, Kevin Chen-Chuan Chang, Yu~Lei, and Bo~Yang.
\newblock Geom-gcn: Geometric graph convolutional networks.
\newblock {\em arXiv preprint arXiv:2002.05287}, 2020.

\bibitem{qiu2022better}
Zengyu Qiu, Xinzhu Ma, Kunlin Yang, Chunya Liu, Jun Hou, Shuai Yi, and Wanli
  Ouyang.
\newblock Better teacher better student: Dynamic prior knowledge for knowledge
  distillation.
\newblock {\em arXiv preprint arXiv:2206.06067}, 2022.

\bibitem{ravanelli2018light}
Mirco Ravanelli, Philemon Brakel, Maurizio Omologo, and Yoshua Bengio.
\newblock Light gated recurrent units for speech recognition.
\newblock {\em IEEE Transactions on Emerging Topics in Computational
  Intelligence}, 2(2):92--102, 2018.

\bibitem{rong2019dropedge}
Yu~Rong, Wenbing Huang, Tingyang Xu, and Junzhou Huang.
\newblock Dropedge: Towards deep graph convolutional networks on node
  classification.
\newblock {\em arXiv preprint arXiv:1907.10903}, 2019.

\bibitem{royden1988real}
Halsey~Lawrence Royden and Patrick Fitzpatrick.
\newblock {\em Real analysis}, volume~32.
\newblock Macmillan New York, 1988.

\bibitem{rubin2020manifold}
Patrick Rubin-Delanchy.
\newblock Manifold structure in graph embeddings.
\newblock {\em Advances in Neural Information Processing Systems},
  33:11687--11699, 2020.

\bibitem{shao2022generalized}
Zhiqi Shao, Andi Han, Dai Shi, Andrey Vasnev, and Junbin Gao.
\newblock Generalized laplacian regularized framelet gcns.
\newblock {\em arXiv preprint arXiv:2210.15092}, 2022.

\bibitem{song2020score}
Yang Song, Jascha Sohl-Dickstein, Diederik~P Kingma, Abhishek Kumar, Stefano
  Ermon, and Ben Poole.
\newblock Score-based generative modeling through stochastic differential
  equations.
\newblock {\em arXiv preprint arXiv:2011.13456}, 2020.

\bibitem{sun2019patient}
Siqi Sun, Yu~Cheng, Zhe Gan, and Jingjing Liu.
\newblock Patient knowledge distillation for bert model compression.
\newblock {\em arXiv preprint arXiv:1908.09355}, 2019.

\bibitem{tailor2020degree}
Shyam~A Tailor, Javier Fernandez-Marques, and Nicholas~D Lane.
\newblock Degree-quant: Quantization-aware training for graph neural networks.
\newblock {\em arXiv preprint arXiv:2008.05000}, 2020.

\bibitem{tang2013universally}
Minh Tang, Daniel~L Sussman, and Carey~E Priebe.
\newblock Universally consistent vertex classification for latent positions
  graphs.
\newblock {\em arXiv preprint arXiv:1212.1182}, 2013.

\bibitem{topping2021understanding}
Jake Topping, Francesco Di~Giovanni, Benjamin~Paul Chamberlain, Xiaowen Dong,
  and Michael~M Bronstein.
\newblock Understanding over-squashing and bottlenecks on graphs via curvature.
\newblock {\em arXiv preprint arXiv:2111.14522}, 2021.

\bibitem{vahdat2021score}
Arash Vahdat, Karsten Kreis, and Jan Kautz.
\newblock Score-based generative modeling in latent space.
\newblock {\em Advances in Neural Information Processing Systems},
  34:11287--11302, 2021.

\bibitem{velivckovic2017graph}
Petar Veli{\v{c}}kovi{\'c}, Guillem Cucurull, Arantxa Casanova, Adriana Romero,
  Pietro Lio, and Yoshua Bengio.
\newblock Graph attention networks.
\newblock {\em arXiv preprint arXiv:1710.10903}, 2017.

\bibitem{wang2021deep}
Jingyi Wang and Zhidong Deng.
\newblock A deep graph wavelet convolutional neural network for semi-supervised
  node classification.
\newblock In {\em 2021 International Joint Conference on Neural Networks
  (IJCNN)}, pages 1--8. IEEE, 2021.

\bibitem{wang2014generalized}
Wei Wang, Yan Huang, Yizhou Wang, and Liang Wang.
\newblock Generalized autoencoder: A neural network framework for
  dimensionality reduction.
\newblock In {\em Proceedings of the IEEE conference on computer vision and
  pattern recognition workshops}, pages 490--497, 2014.

\bibitem{wu2019simplifying}
Felix Wu, Amauri Souza, Tianyi Zhang, Christopher Fifty, Tao Yu, and Kilian
  Weinberger.
\newblock Simplifying graph convolutional networks.
\newblock In {\em International Conference on Machine Learning}, pages
  6861--6871. PMLR, 2019.

\bibitem{wu2020comprehensive}
Zonghan Wu, Shirui Pan, Fengwen Chen, Guodong Long, Chengqi Zhang, and S~Yu
  Philip.
\newblock A comprehensive survey on graph neural networks.
\newblock {\em IEEE transactions on neural networks and learning systems},
  32(1):4--24, 2020.

\bibitem{xu2018graph}
Bingbing Xu, Huawei Shen, Qi~Cao, Yunqi Qiu, and Xueqi Cheng.
\newblock Graph wavelet neural network.
\newblock In {\em International Conference on Learning Representations}, 2018.

\bibitem{xu2018powerful}
Keyulu Xu, Weihua Hu, Jure Leskovec, and Stefanie Jegelka.
\newblock How powerful are graph neural networks?
\newblock {\em arXiv preprint arXiv:1810.00826}, 2018.

\bibitem{xu2020deep}
Ting-Bing Xu and Cheng-Lin Liu.
\newblock Deep neural network self-distillation exploiting data representation
  invariance.
\newblock {\em IEEE Transactions on Neural Networks and Learning Systems},
  33(1):257--269, 2020.

\bibitem{yang2022quasi}
Mengxi Yang, Xuebin Zheng, Jie Yin, and Junbin Gao.
\newblock Quasi-framelets: Another improvement to graph neural networks.
\newblock {\em arXiv:2201.04728}, 2022.

\bibitem{YangZhouYinGao2022}
Mengxi Yang, Xuebin Zhou, Jie Yin, and Junbin Gao.
\newblock Quasi-framelets: Another improvement to spectral graph neural
  networks.
\newblock {\em arXiv:2201.04728}, 2022.

\bibitem{yang2020distilling}
Yiding Yang, Jiayan Qiu, Mingli Song, Dacheng Tao, and Xinchao Wang.
\newblock Distilling knowledge from graph convolutional networks.
\newblock In {\em Proceedings of the IEEE/CVF Conference on Computer Vision and
  Pattern Recognition}, pages 7074--7083, 2020.

\bibitem{ye2021distillation}
Jianming Ye, Jingdong Wang, and Shiliang Zhang.
\newblock Distillation-guided residual learning for binary convolutional neural
  networks.
\newblock {\em IEEE Transactions on Neural Networks and Learning Systems},
  33(12):7765--7777, 2021.

\bibitem{zhang2021self}
Linfeng Zhang, Chenglong Bao, and Kaisheng Ma.
\newblock Self-distillation: Towards efficient and compact neural networks.
\newblock {\em IEEE Transactions on Pattern Analysis and Machine Intelligence},
  44(8):4388--4403, 2021.

\bibitem{zhang2019your}
Linfeng Zhang, Jiebo Song, Anni Gao, Jingwei Chen, Chenglong Bao, and Kaisheng
  Ma.
\newblock Be your own teacher: Improve the performance of convolutional neural
  networks via self distillation.
\newblock In {\em Proceedings of the IEEE/CVF International Conference on
  Computer Vision}, pages 3713--3722, 2019.

\bibitem{zhang2022ms}
Mo~Zhang, Bin Dong, and Quanzheng Li.
\newblock Ms-gwnn: multi-scale graph wavelet neural network for breast cancer
  diagnosis.
\newblock In {\em 2022 IEEE 19th International Symposium on Biomedical Imaging
  (ISBI)}, pages 1--5. IEEE, 2022.

\bibitem{zhang2021graph}
Shichang Zhang, Yozen Liu, Yizhou Sun, and Neil Shah.
\newblock Graph-less neural networks: Teaching old mlps new tricks via
  distillation.
\newblock {\em arXiv preprint arXiv:2110.08727}, 2021.

\bibitem{zhang2018deep}
Ying Zhang, Tao Xiang, Timothy~M Hospedales, and Huchuan Lu.
\newblock Deep mutual learning.
\newblock In {\em Proceedings of the IEEE conference on computer vision and
  pattern recognition}, pages 4320--4328, 2018.

\bibitem{zheng2021framelets}
Xuebin Zheng, Bingxin Zhou, Junbin Gao, Yu~Guang Wang, Pietro Li{\'o}, Ming Li,
  and Guido Mont{\'u}far.
\newblock How framelets enhance graph neural networks.
\newblock {\em arXiv preprint arXiv:2102.06986}, 2021.

\bibitem{zheng2020mathnet}
Xuebin Zheng, Bingxin Zhou, Ming Li, Yu~Guang Wang, and Junbin Gao.
\newblock Mathnet: Haar-like wavelet multiresolution-analysis for graph
  representation and learning.
\newblock {\em arXiv preprint arXiv:2007.11202}, 2020.

\bibitem{zheng2022decimated}
Xuebin Zheng, Bingxin Zhou, Yu~Guang Wang, and Xiaosheng Zhuang.
\newblock Decimated framelet system on graphs and fast g-framelet transforms.
\newblock {\em Journal of Machine Learning Research}, 23(18):1--68, 2022.

\bibitem{zhou2021graph}
Bingxin Zhou, Ruikun Li, Xuebin Zheng, Yu~Guang Wang, and Junbin Gao.
\newblock Graph denoising with framelet regularizer.
\newblock {\em arXiv preprint arXiv:2111.03264}, 2021.

\bibitem{zhou2021spectral}
Bingxin Zhou, Xinliang Liu, Yuehua Liu, Yunying Huang, Pietro Lio, and YuGuang
  Wang.
\newblock Spectral transform forms scalable transformer.
\newblock {\em arXiv:2111.07602}, 2021.

\bibitem{zhou2021accelerating}
Hongkuan Zhou, Ajitesh Srivastava, Hanqing Zeng, Rajgopal Kannan, and Viktor
  Prasanna.
\newblock Accelerating large scale real-time gnn inference using channel
  pruning.
\newblock {\em arXiv preprint arXiv:2105.04528}, 2021.

\bibitem{zhu2020beyond}
Jiong Zhu, Yujun Yan, Lingxiao Zhao, Mark Heimann, Leman Akoglu, and Danai
  Koutra.
\newblock Beyond homophily in graph neural networks: Current limitations and
  effective designs.
\newblock {\em Advances in Neural Information Processing Systems},
  33:7793--7804, 2020.

\bibitem{zhu2020knowledge}
Mingrui Zhu, Jie Li, Nannan Wang, and Xinbo Gao.
\newblock Knowledge distillation for face photo--sketch synthesis.
\newblock {\em IEEE Transactions on Neural Networks and Learning Systems},
  33(2):893--906, 2020.

\bibitem{zhu2021student}
Yichen Zhu and Yi~Wang.
\newblock Student customized knowledge distillation: Bridging the gap between
  student and teacher.
\newblock In {\em Proceedings of the IEEE/CVF International Conference on
  Computer Vision}, pages 5057--5066, 2021.

\end{thebibliography}


\begin{thebibliography}{10}

\bibitem{alon2020bottleneck}
Uri Alon and Eran Yahav.
\newblock On the bottleneck of graph neural networks and its practical
  implications.
\newblock In {\em International Conference on Learning Representations}, 2020.

\bibitem{avelar2019discrete}
Pedro~HC Avelar, Anderson~R Tavares, Marco Gori, and Luis~C Lamb.
\newblock Discrete and continuous deep residual learning over graphs.
\newblock {\em arXiv:1911.09554}, 2019.

\bibitem{bo2021beyond}
Deyu Bo, Xiao Wang, Chuan Shi, and Huawei Shen.
\newblock Beyond low-frequency information in graph convolutional networks.
\newblock In {\em AAAI Conference on Artificial Intelligence}, volume~35, pages
  3950--3957, 2021.

\bibitem{bodnar2022neural}
Cristian Bodnar, Francesco Di~Giovanni, Benjamin~Paul Chamberlain, Pietro
  Li{\`o}, and Michael~M Bronstein.
\newblock Neural sheaf diffusion: A topological perspective on heterophily and
  oversmoothing in gnns.
\newblock {\em arXiv:2202.04579}, 2022.

\bibitem{cai2020note}
Chen Cai and Yusu Wang.
\newblock A note on over-smoothing for graph neural networks.
\newblock {\em arXiv:2006.13318}, 2020.

\bibitem{chamberlain2021grand}
Ben Chamberlain, James Rowbottom, Maria~I Gorinova, Michael Bronstein, Stefan
  Webb, and Emanuele Rossi.
\newblock Grand: Graph neural diffusion.
\newblock In {\em International Conference on Machine Learning}, pages
  1407--1418. PMLR, 2021.

\bibitem{chamberlain2021beltrami}
Benjamin Chamberlain, James Rowbottom, Davide Eynard, Francesco Di~Giovanni,
  Xiaowen Dong, and Michael Bronstein.
\newblock Beltrami flow and neural diffusion on graphs.
\newblock {\em Advances in Neural Information Processing Systems},
  34:1594--1609, 2021.

\bibitem{chendirichlet}
Jialin Chen, Yuelin Wang, Cristian Bodnar, Pietro Li{\`o}, and Yu~Guang Wang.
\newblock Dirichlet energy enhancement of graph neural networks by framelet
  augmentation.
\newblock {\em github}, 2022.

\bibitem{chen2018fastgcn}
Jie Chen, Tengfei Ma, and Cao Xiao.
\newblock Fast{GCN}: Fast learning with graph convolutional networks via
  importance sampling.
\newblock In {\em International Conference on Learning Representations}, 2018.

\bibitem{chen2020simple}
Ming Chen, Zhewei Wei, Zengfeng Huang, Bolin Ding, and Yaliang Li.
\newblock Simple and deep graph convolutional networks.
\newblock In {\em International Conference on Machine Learning}, pages
  1725--1735. PMLR, 2020.

\bibitem{chen2022optimization}
Qi~Chen, Yifei Wang, Yisen Wang, Jiansheng Yang, and Zhouchen Lin.
\newblock Optimization-induced graph implicit nonlinear diffusion.
\newblock In {\em International Conference on Machine Learning}, pages
  3648--3661. PMLR, 2022.

\bibitem{chen2018neural}
Ricky~TQ Chen, Yulia Rubanova, Jesse Bettencourt, and David~K Duvenaud.
\newblock Neural ordinary differential equations.
\newblock {\em Advances in Neural Information Processing Systems}, 31, 2018.

\bibitem{chung1997spectral}
Fan~RK Chung.
\newblock {\em Spectral graph theory}, volume~92.
\newblock American Mathematical Soc., 1997.

\bibitem{cui2019traffic}
Zhiyong Cui, Kristian Henrickson, Ruimin Ke, and Yinhai Wang.
\newblock Traffic graph convolutional recurrent neural network: A deep learning
  framework for network-scale traffic learning and forecasting.
\newblock {\em IEEE Transactions on Intelligent Transportation Systems},
  21(11):4883--4894, 2019.

\bibitem{defferrard2016convolutional}
Micha{\"e}l Defferrard, Xavier Bresson, and Pierre Vandergheynst.
\newblock Convolutional neural networks on graphs with fast localized spectral
  filtering.
\newblock {\em Advances in Neural Information Processing Systems}, 29, 2016.

\bibitem{di2022graph}
Francesco Di~Giovanni, James Rowbottom, Benjamin~P Chamberlain, Thomas
  Markovich, and Michael~M Bronstein.
\newblock Graph neural networks as gradient flows.
\newblock {\em arXiv:2206.10991}, 2022.

\bibitem{dong2017sparse}
Bin Dong.
\newblock Sparse representation on graphs by tight wavelet frames and
  applications.
\newblock {\em Applied and Computational Harmonic Analysis}, 42(3):452--479,
  2017.

\bibitem{duvenaud2015convolutional}
David~K Duvenaud, Dougal Maclaurin, Jorge Iparraguirre, Rafael Bombarell,
  Timothy Hirzel, Al{\'a}n Aspuru-Guzik, and Ryan~P Adams.
\newblock Convolutional networks on graphs for learning molecular fingerprints.
\newblock {\em Advances in Neural Information Processing Systems}, 28, 2015.

\bibitem{eliasof2021pde}
Moshe Eliasof, Eldad Haber, and Eran Treister.
\newblock {PDE-GCN}: Novel architectures for graph neural networks motivated by
  partial differential equations.
\newblock {\em Advances in Neural Information Processing Systems},
  34:3836--3849, 2021.

\bibitem{gasteiger2018predict}
Johannes Gasteiger, Aleksandar Bojchevski, and Stephan G{\"u}nnemann.
\newblock Predict then propagate: Graph neural networks meet personalized
  pagerank.
\newblock In {\em International Conference on Learning Representations}, 2018.

\bibitem{gilmer2017neural}
Justin Gilmer, Samuel~S Schoenholz, Patrick~F Riley, Oriol Vinyals, and
  George~E Dahl.
\newblock Neural message passing for quantum chemistry.
\newblock In {\em International Conference on Machine Learning}, pages
  1263--1272. PMLR, 2017.

\bibitem{hamilton2017inductive}
Will Hamilton, Zhitao Ying, and Jure Leskovec.
\newblock Inductive representation learning on large graphs.
\newblock {\em Advances in Neural Information Processing Systems}, 30, 2017.

\bibitem{hammond2011wavelets}
David~K Hammond, Pierre Vandergheynst, and R{\'e}mi Gribonval.
\newblock Wavelets on graphs via spectral graph theory.
\newblock {\em Applied and Computational Harmonic Analysis}, 30(2):129--150,
  2011.

\bibitem{hansen2020sheaf}
Jakob Hansen and Thomas Gebhart.
\newblock Sheaf neural networks.
\newblock In {\em NeurIPS Workshop on Topological Data Analysis and Beyond},
  2020.

\bibitem{he2016deep}
Kaiming He, Xiangyu Zhang, Shaoqing Ren, and Jian Sun.
\newblock Deep residual learning for image recognition.
\newblock In {\em Computer Vision and Pattern Recognition}, pages 770--778,
  2016.

\bibitem{he2021bernnet}
Mingguo He, Zhewei Wei, Hongteng Xu, et~al.
\newblock Bernnet: Learning arbitrary graph spectral filters via bernstein
  approximation.
\newblock {\em Advances in Neural Information Processing Systems},
  34:14239--14251, 2021.

\bibitem{hu2019exploring}
Shell~Xu Hu, Sergey Zagoruyko, and Nikos Komodakis.
\newblock Exploring weight symmetry in deep neural networks.
\newblock {\em Computer Vision and Image Understanding}, 187:102786, 2019.

\bibitem{kipf2016semi}
Thomas~N Kipf and Max Welling.
\newblock Semi-supervised classification with graph convolutional networks.
\newblock {\em arXiv:1609.02907}, 2016.

\bibitem{klicpera2019diffusion}
Johannes Klicpera, Stefan Wei{\ss}enberger, and Stephan G{\"u}nnemann.
\newblock Diffusion improves graph learning.
\newblock In {\em International Conference on Neural Information Processing
  Systems}, pages 13366--13378, 2019.

\bibitem{li2019deepgcns}
Guohao Li, Matthias Muller, Ali Thabet, and Bernard Ghanem.
\newblock Deepgcns: Can gcns go as deep as cnns?
\newblock In {\em International Conference on Computer Vision}, pages
  9267--9276, 2019.

\bibitem{li2020fast}
Ming Li, Zheng Ma, Yu~Guang Wang, and Xiaosheng Zhuang.
\newblock Fast {H}aar transforms for graph neural networks.
\newblock {\em Neural Networks}, 128:188--198, 2020.

\bibitem{li2018deeper}
Qimai Li, Zhichao Han, and Xiao-Ming Wu.
\newblock Deeper insights into graph convolutional networks for semi-supervised
  learning.
\newblock In {\em AAAI Conference on Artificial Intelligence}, 2018.

\bibitem{maron2019provably}
Haggai Maron, Heli Ben-Hamu, Hadar Serviansky, and Yaron Lipman.
\newblock Provably powerful graph networks.
\newblock {\em Advances in Neural Information Processing Systems}, 32, 2019.

\bibitem{ni2019community}
Chien-Chun Ni, Yu-Yao Lin, Feng Luo, and Jie Gao.
\newblock Community detection on networks with {R}icci flow.
\newblock {\em Scientific reports}, 9(1):1--12, 2019.

\bibitem{nt2019revisiting}
Hoang Nt and Takanori Maehara.
\newblock Revisiting graph neural networks: All we have is low-pass filters.
\newblock {\em arXiv:1905.09550}, 2019.

\bibitem{oono2019graph}
Kenta Oono and Taiji Suzuki.
\newblock Graph neural networks exponentially lose expressive power for node
  classification.
\newblock In {\em International Conference on Learning Representations}, 2019.

\bibitem{pei2019geom}
Hongbin Pei, Bingzhe Wei, Kevin Chen-Chuan Chang, Yu~Lei, and Bo~Yang.
\newblock Geom-{GCN}: Geometric graph convolutional networks.
\newblock In {\em International Conference on Learning Representations}, 2019.

\bibitem{poli2019graph}
Michael Poli, Stefano Massaroli, Junyoung Park, Atsushi Yamashita, Hajime
  Asama, and Jinkyoo Park.
\newblock Graph neural ordinary differential equations.
\newblock {\em arXiv:1911.07532}, 2019.

\bibitem{rusch2022graph}
T~Konstantin Rusch, Ben Chamberlain, James Rowbottom, Siddhartha Mishra, and
  Michael Bronstein.
\newblock Graph-coupled oscillator networks.
\newblock In {\em International Conference on Machine Learning}, pages
  18888--18909. PMLR, 2022.

\bibitem{thorpe2021grand}
Matthew Thorpe, Tan~Minh Nguyen, Hedi Xia, Thomas Strohmer, Andrea Bertozzi,
  Stanley Osher, and Bao Wang.
\newblock {GRAND++}: Graph neural diffusion with a source term.
\newblock In {\em International Conference on Learning Representations}, 2021.

\bibitem{topping2021understanding}
Jake Topping, Francesco Di~Giovanni, Benjamin~Paul Chamberlain, Xiaowen Dong,
  and Michael~M Bronstein.
\newblock Understanding over-squashing and bottlenecks on graphs via curvature.
\newblock In {\em International Conference on Learning Representations}, 2021.

\bibitem{velivckovic2018graph}
Petar Veli{\v{c}}kovi{\'c}, Guillem Cucurull, Arantxa Casanova, Adriana Romero,
  Pietro Li{\`o}, and Yoshua Bengio.
\newblock Graph attention networks.
\newblock In {\em International Conference on Learning Representations}, 2018.

\bibitem{wang2021deep}
Jingyi Wang and Zhidong Deng.
\newblock A deep graph wavelet convolutional neural network for semi-supervised
  node classification.
\newblock In {\em 2021 International Joint Conference on Neural Networks
  (IJCNN)}, pages 1--8. IEEE, 2021.

\bibitem{wang2022powerful}
Xiyuan Wang and Muhan Zhang.
\newblock How powerful are spectral graph neural networks.
\newblock In {\em International Conference on Machine Learning}, 2022.

\bibitem{wang2021dissecting}
Yifei Wang, Yisen Wang, Jiansheng Yang, and Zhouchen Lin.
\newblock Dissecting the diffusion process in linear graph convolutional
  networks.
\newblock {\em Advances in Neural Information Processing Systems},
  34:5758--5769, 2021.

\bibitem{wu2019simplifying}
Felix Wu, Amauri Souza, Tianyi Zhang, Christopher Fifty, Tao Yu, and Kilian
  Weinberger.
\newblock Simplifying graph convolutional networks.
\newblock In {\em International Conference on Machine Learning}, pages
  6861--6871. PMLR, 2019.

\bibitem{xhonneux2020continuous}
Louis-Pascal Xhonneux, Meng Qu, and Jian Tang.
\newblock Continuous graph neural networks.
\newblock In {\em International Conference on Machine Learning}, pages
  10432--10441. PMLR, 2020.

\bibitem{xu2018graph}
Bingbing Xu, Huawei Shen, Qi~Cao, Yunqi Qiu, and Xueqi Cheng.
\newblock Graph wavelet neural network.
\newblock In {\em International Conference on Learning Representations}, 2018.

\bibitem{xu2018powerful}
Keyulu Xu, Weihua Hu, Jure Leskovec, and Stefanie Jegelka.
\newblock How powerful are graph neural networks?
\newblock In {\em International Conference on Learning Representations}, 2018.

\bibitem{yang2022quasi}
Mengxi Yang, Xuebin Zheng, Jie Yin, and Junbin Gao.
\newblock Quasi-framelets: Another improvement to graph neural networks.
\newblock {\em arXiv:2201.04728}, 2022.

\bibitem{zheng2021framelets}
Xuebin Zheng, Bingxin Zhou, Junbin Gao, Yuguang Wang, Pietro Li{\'o}, Ming Li,
  and Guido Montufar.
\newblock How framelets enhance graph neural networks.
\newblock In {\em International Conference on Machine Learning}, pages
  12761--12771. PMLR, 2021.

\bibitem{zheng2020mathnet}
Xuebin Zheng, Bingxin Zhou, Ming Li, Yu~Guang Wang, and Junbin Gao.
\newblock {MathNet}: Haar-like wavelet multiresolution-analysis for graph
  representation and learning.
\newblock {\em arXiv:2007.11202}, 2020.

\bibitem{zheng2022decimated}
Xuebin Zheng, Bingxin Zhou, Yu~Guang Wang, and Xiaosheng Zhuang.
\newblock Decimated framelet system on graphs and fast g-framelet transforms.
\newblock {\em Journal of Machine Learning Research}, 23:18--1, 2022.

\bibitem{zhou2022graph}
Bingxin Zhou, Yuanhong Jiang, Yu~Guang Wang, Jingwei Liang, Junbin Gao, Shirui
  Pan, and Xiaoqun Zhang.
\newblock Graph neural network for local corruption recovery.
\newblock {\em arXiv:2202.04936}, 2022.

\bibitem{zhou2021graph}
Bingxin Zhou, Ruikun Li, Xuebin Zheng, Yu~Guang Wang, and Junbin Gao.
\newblock Graph denoising with framelet regularizer.
\newblock {\em arXiv:2111.03264}, 2021.

\bibitem{zhou2021spectral}
Bingxin Zhou, Xinliang Liu, Yuehua Liu, Yunying Huang, Pietro Lio, and YuGuang
  Wang.
\newblock Spectral transform forms scalable transformer.
\newblock {\em arXiv:2111.07602}, 2021.

\bibitem{zhu2020beyond}
Jiong Zhu, Yujun Yan, Lingxiao Zhao, Mark Heimann, Leman Akoglu, and Danai
  Koutra.
\newblock Beyond homophily in graph neural networks: Current limitations and
  effective designs.
\newblock {\em Advances in Neural Information Processing Systems},
  33:7793--7804, 2020.

\bibitem{zou2022simple}
Chunya Zou, Andi Han, Lequan Lin, and Junbin Gao.
\newblock A simple yet effective {SVD-GCN} for directed graphs.
\newblock {\em arXiv:2205.09335}, 2022.

\end{thebibliography}

\end{document}